\documentclass{tlp}

\usepackage{helvet,times,courier}
\usepackage[utf8]{inputenc}      
\usepackage[T1]{fontenc}         
\usepackage[final]{microtype}    
\usepackage{amsmath,amsfonts,mathtools}    
\usepackage{xspace}              
\usepackage{tikz}                
\usepackage[
  ruled
, vlined
, linesnumbered]{algorithm2e}    
\usepackage{todonotes}           
\usepackage{alltt}				 

\usepackage{subfigure}
\usepackage{booktabs}
\newtheorem{proposition}{Proposition}
\newtheorem{definition}{Definition}
\newtheorem{example}{Example}
\newtheorem{theorem}{Theorem}

\newcommand{\myParagraph}[1]{\smallskip\noindent\textit{#1}}

\usetikzlibrary{shapes,arrows,shadows,fit,positioning,calc}
\tikzstyle{block} = [draw, rectangle, minimum height=1.8em, minimum width=9em, node distance=3.5em, transform shape]

\DontPrintSemicolon
\SetKwInput{Input}{input}

\newcommand{\wasp}{\textsc{wasp}\xspace}
\newcommand{\dwasp}{\textsc{dwasp}\xspace}
\newcommand{\dwaspgui}{\textsc{dwasp-gui}\xspace}
\newcommand{\gringo}{\textsc{gringo}\xspace}
\newcommand{\gringowrapper}{\textsc{gringo-wrapper}\xspace}
\newcommand{\aspide}{\textsc{aspide}\xspace}

\newcommand{\pr}{\Pi}                        
\newcommand{\dpr}[2][]{%
	\ifthenelse{\equal{#1}{}}
		{\Delta_{#2}}
		{\Delta^{#1}_{#2}}
	}       
\newcommand{\gpr}{\pr^G}                     
\newcommand{\bg}{\mathcal{B}}                
\newcommand{\debugground}[1]{\_debug(#1)}       
\newcommand{\debugempty}{\_debug(\ensuremath{\cdot})\xspace}       
\newcommand{\debug}[2]{\_debug(#1,#2)}       
\newcommand{\supportempty}{\_support(\ensuremath{\cdot})}
\newcommand{\support}[1]{\_support(#1)}      

\newcommand{\dnot}{\ensuremath{\raise.17ex\hbox{\ensuremath{\scriptstyle\mathtt{\sim}}}}\xspace}                   
\newcommand{\imp}{\leftarrow}                

\title[Debugging Non-Ground ASP Programs]
{Debugging Non-Ground ASP Programs: \\ Technique and Graphical Tools\footnote{This paper includes parts and significantly extends our previous work 
 \cite{DBLP:conf/lpnmr/DodaroGMRS15}.
}
}

\author[C. Dodaro et al.]{
	CARMINE DODARO\\
	DIBRIS, University of Genova, Italy\\
	\email{dodaro@dibris.unige.it}\\
	\and
	PHILIP GASTEIGER\\
	Alpen-Adria-Universit\"at Klagenfurt, Austria \\
	\email{philip.gasteiger@gmail.com}\\
	\and
	KRISTIAN REALE, FRANCESCO RICCA\\
	Department of Mathematics and Computer Science, University of Calabria, Italy \\
	\email{\{reale, ricca\}@mat.unical.it}\\
	\and
	KONSTANTIN SCHEKOTIHIN\\
	Alpen-Adria-Universit\"at Klagenfurt, Austria \\
	\email{konstantin.schekotihin@aau.at}\\
}

\jdate{March 2003}
\pubyear{2003}
\pagerange{\pageref{firstpage}--\pageref{lastpage}}
\doi{}

\begin{document}
\maketitle

\begin{abstract}
	Answer Set Programming (ASP) is one of the major declarative programming paradigms in the area of logic programming and non-monotonic reasoning.
	Despite that ASP features a simple syntax and an intuitive semantics, errors are common during the development of ASP programs.	
	In this paper we propose a novel debugging approach allowing for interactive localization of bugs in non-ground programs. 
	The new approach points the user directly to a set of non-ground rules involved in the bug, which might be refined (up to the point in which the bug is easily identified) by asking the programmer a sequence of questions on an expected answer set. 
	The approach has been implemented on top of the ASP solver \wasp. 
	The resulting debugger has been complemented by a user-friendly graphical interface, and integrated in \aspide, a rich IDE for answer set programs. 
	In addition, an empirical analysis shows that the new debugger is not affected by the grounding blowup limiting the application of previous approaches based on meta-programming.
	Under consideration in Theory and Practice of Logic Programming (TPLP).
	
\end{abstract}

\section{Introduction}
\label{sec:introduction}
Answer Set Programming (ASP)~\cite{DBLP:journals/cacm/BrewkaET11,DBLP:journals/ngc/GelfondL91} is a declarative programming paradigm proposed in the area of logic programming and non-monotonic reasoning. 
ASP features an expressive language that can be used to model computational problems of comparatively high complexity~\cite{DBLP:journals/tods/EiterGM97} often in a rather compact way.
The availability of high-performance implementations~\cite{DBLP:conf/lpnmr/GebserMR15,DBLP:journals/aim/LierlerMR16,DBLP:journals/ai/CalimeriGMR16,DBLP:conf/aaai/GebserMR16} made ASP a valuable tool for developing complex applications in several research areas~\cite{DBLP:journals/aim/ErdemGL16}, including Artificial Intelligence~\cite{DBLP:conf/lpnmr/BalducciniGWN01,DBLP:conf/cpaior/AschingerDFGJRT11,DBLP:conf/rr/DodaroLNR15,DBLP:journals/fuin/AbseherGMSW16,DBLP:journals/tplp/DodaroGLMRS16}, Hydroinformatics~\cite{DBLP:journals/logcom/GavanelliNP15}, Nurse Scheduling~\cite{DBLP:conf/aiia/AlvianoDM17}, and Bioinformatics~\cite{DBLP:journals/tplp/ErdemO15,DBLP:journals/tplp/KoponenOJS15,DBLP:journals/tplp/GebserSTV11}, to mention a few.
Especially the development of real-world applications outlined the advantages of ASP from a software engineering viewpoint.
Namely, ASP programs are flexible, compact, extensible and easy to maintain~\cite{DBLP:conf/birthday/GrassoLMR11}.

Although the basic syntax of ASP is not particularly difficult, one of the most tedious and time-consuming  programming tasks is the identification of (even trivial) faults in a program.
For this reason, several methodologies and tools have been proposed in the last few years for debugging ASP programs~\cite{DBLP:conf/asp/BrainV05,DBLP:journals/tplp/PontelliSE09,DBLP:conf/lpnmr/OetschPT11,DBLP:conf/aaai/GebserPST08,DBLP:journals/tplp/OetschPT10,DBLP:conf/aaai/Shchekotykhin15} with the goal of making the process of developing logic programs more rapid and comfortable.
Given a faulty ASP program $\pr$ and a set of atoms $I$ representing an interpretation of the program, these approaches find an explanation why $I$ \emph{is not an answer set} of $\pr$.

The most prominent debugging approaches \cite{DBLP:conf/aaai/GebserPST08,DBLP:journals/tplp/OetschPT10} apply the notion of meta-programming that uses ASP itself to debug a faulty ASP program.
The basic idea of the meta-programming method is to convert the inputs into a program over a meta-language -- a reified program --  and then execute it together with a debugging program.
The latter finds causes of a fault, where each cause is encoded by specific atoms in an answer set of the debugging program.
However, reification-based debuggers have some issues that may make them either difficult to apply or even inapplicable in some cases.
The main issue, also observed in~\cite{DBLP:journals/tplp/OetschPT10}, is related to the computation of answer sets of the debugging program.
Namely, the grounding step of the solving process might produce ground instantiations of the debugging program, which sizes are exponentially larger than the input. 
This problem is intrinsic in the meta-programming approach, as it requires the ground debugging program to comprise a set of ground rules encoding all possible explanations of faults in the input program.
Moreover, even if answer sets of the reified program can be computed, their number can overwhelm a user, who has to analyze all explanations manually in order to find a true cause of the problem.
A recent approach suggested in \cite{DBLP:conf/aaai/Shchekotykhin15} allows for a partial resolution of this issue, but it is applicable only to ground programs, which often contains a large number of rules and, consequently, are hard to understand and to debug.

In this paper we propose a novel debugging approach allowing for interactive localization of bugs in non-ground ASP programs.
The new approach points the user directly to a set of non-ground rules involved in the bug, which might be refined (up to the point in which the bug is easily identified) by asking the programmer a sequence of questions on an expected answer set. 
Roughly, the suggested approach can be described as follows:
First, given a non-ground program $\pr$ the debugger generates a debugging program $\dpr{\pr}$ by adding marker atoms used later to match results of the debugger with rules of $\pr$; 
Next, the debugging program is grounded and passed, together with (the second input) an interpretation $I$, to a specifically-modified version of an ASP solver that determines the set of rules that are possible reasons for the bug.
As previous approaches, our debugger allows for a uniform treatment of \textit{over-constrained} and \textit{missing support} faults (see Section~\ref{sec:debugging-approach}).
%
%
Since the number of reasons might be large (up to covering the entire program), we help the user to find the \emph{guilty rules} by applying an automated refining technique.
In particular, the debugger automatically generates a sequence of queries, requiring the user to answer whether a set of literals must be included or not in an expected answer set.
This additional knowledge is then injected in the system and the process is repeated up to a point in which the (non-ground) rules causing a fault can easily be identified.

The approach has been implemented in the \dwasp debugger, that combines the grounder \gringo~\cite{DBLP:conf/lpnmr/GebserKKS11} with an extension of the ASP solver \wasp~\cite{DBLP:conf/lpnmr/AlvianoDLR15}.
The resulting implementation can be used via command-line interface.
%
%
In order to further ease the task of debugging logic programs, and appeal to those users that prefer graphical interfaces, we also developed a graphical user interface for the \dwasp debugger, called \dwaspgui, and a plugin connector for the Integrated Development Environment (IDE) \aspide.
An important reason to choose \aspide over other ASP IDEs, like Sealion \cite{DBLP:journals/tplp/BusoniuOPST13}, is its richer support for test-driven development of ASP programs~\cite{DBLP:conf/inap/FebbraroLRR11}.
The test-driven development process~\cite{DBLP:conf/xpu/FraserBCMNP03} requires the repetition of a very short development cycle in which requirements are encoded as specific test cases (assessing a possibly small unit of the program), and the program is assessed against tests, and possibly fixed or improved. The cycle is then repeated to push forward the functionality, until the program satisfies all the requirements. The rapid identification of the cause of a failing test case is fundamental for test-driven development platforms.
In \aspide a user can naturally define test cases comprising inputs and expected (non-)outputs of a solver when applied to compute answer sets of a developed ASP program. 
Services of the IDE allow a user to execute the test cases and to generate a report of the results.
Thus, for a reported failed test case, the plug-in connector automatically generates and forwards to \dwasp all required inputs, configures and executes of debugging tasks relevant to the studied test case. 
In this way the \dwasp debugger synergistically works with the test-driven framework of \aspide resulting in a more complete test-driven development environment.
Overall, the combination provides an intuitive user-experience with \aspide similar to the one of modern IDEs for software development with imperative languages.

To summarize, the paper makes the following contributions:%
\footnote{
	This paper is an extended and improved verison of \cite{DBLP:conf/lpnmr/DodaroGMRS15} featuring the following improvements: $(i)$ we provide a new formal description of the debugging approach, and $(ii)$ we prove some of its formal properties; $(iii)$ we extend the approach with ``missing support'' faults; $(iv)$ we implement the \dwaspgui; $(v)$ we extend \aspide with the new debugger. $(vi)$ we perform a usability testing assessment on students of a course on ASP.
}
\begin{enumerate}
	\item We present an interactive and efficient debugging technique for non-ground ASP programs (see Section~\ref{sec:debugging-approach}).
	\item We implement a tool, called \dwasp, for debugging \emph{non-ground} ASP programs supporting all syntax of the latest ASP-Core~(see Section~\ref{sec:debugger}).
	\item We suggest a new graphical debugging interface for ASP programs, called \dwaspgui, based on \dwasp that improves the user-experience of the debugger~(see Section~\ref{sec:system-description:dwaspgui}).
	\item We integrate \dwaspgui over a new plug-in connector with \aspide (see Section~\ref{sec:system-description:aspide}).
	\item We compare our tool with state-of-the-art debuggers based on meta-programming approaches~\cite{DBLP:conf/aaai/GebserPST08,DBLP:journals/tplp/OetschPT10}, and show that our approach is basically unaffected by the combinatorial blow-up which limits the performance of meta-programming approaches (cfr.~\cite{DBLP:journals/tplp/OetschPT10}) (see Section~\ref{sec:experiments}).
	\item We report the results of an assessment of usability and appreciation of the debugger obtained by running a user experience experiment involving students of a course on ASP (see Section~\ref{sec:students}).
\end{enumerate}


\section{Answer Set Programming}
\label{sec:preliminaries}
In this section we overview ASP focusing on preliminary notions that are required for describing the debugging approach implemented in \dwasp.
The reader is referred to~\cite{DBLP:books/daglib/0040913,DBLP:series/synthesis/2012Gebser} for a more comprehensive presentation of ASP.

\myParagraph{Syntax.} 
A program $\pr$ is a finite set of rules of the form
\begin{equation}
	a_1 \lor \ldots \lor a_m \imp l_1, \ldots, l_n
	\label{eq:rule}
\end{equation}
where $a_1, \ldots, a_m$ are atoms and $l_1, \ldots, l_n$ are literals for $m, n \geq 0$.
In particular, an \emph{atom} is an expression of the form $p(t_1, \ldots, t_k)$, where $p$ is a predicate symbol and $t_1, \ldots, t_k$ are \emph{terms}. 
Terms are alphanumeric strings, and are distinguished in variables and constants. 
According to the Prolog's convention, only variables start with an uppercase letter.
A \emph{literal} is an atom $a_i$ (positive) or its negation $\dnot a_i$ (negative), where $\dnot$ denotes the \emph{negation as failure}.
An atom, literal, or rule is called \emph{ground}, if it contains no variable.
The complement of a literal $l$ is denoted by $\overline{l}$.
In particular, given atom $a$ it holds that $\overline{a} = \dnot a$ and $\overline{\dnot a} = a$. 
Moreover, the complement for a set of literals $L$ is $\overline{L} := \{ \overline{l} \mid l \in L \}$.
%
Given a rule $r$ of the form \eqref{eq:rule}, the set of atoms $H(r) = \{ a_1, \ldots, a_m \}$ is called \emph{head} and the set of literals $B(r) = \{ l_1, \ldots, l_n \}$ is called \emph{body}. 
Moreover, $B(r)$ can be partitioned into the sets $B^+(r)$ and $B^-(r)$ comprising the positive and negative body literals, respectively.
A rule $r$ is called \emph{fact} if $|H(r)| = 1$ and $B(r) = \emptyset$ and \emph{constraint} if $H(r) = \emptyset$.
Every rule $r \in \pr$ must be \emph{safe}, i.e. each variable of $r$ must occur in at least one positive literal of $B^+(r)$.
In the following, we will also use \textit{choice rules} of the form $\{a\}$, where $a$ is a ground atom.
A choice rule $\{a\}$ is hereafter considered as a syntactic shortcut for the rule $a \vee a_F \imp$, where $a_F$ is a fresh new atom not appearing elsewhere in the program.

\myParagraph{Semantics.}
Let $\pr$ be an ASP program, the Herbrand Universe $U_\pr$, and the Herbrand base $B_\pr$ are defined as usual.
The semantics of an ASP program is given in terms of the answer sets of its ground instantiation.
The ground instantiation of $\pr$, denoted by $\gpr$, is the ground program obtained by properly substituting all variables occurring in rules from $\pr$ with elements of $U_\pr$.
An \emph{interpretation} is a set of ground atoms $I \subseteq B_\pr$.
Relation $\models$ is inductively defined as follows:
for $a \in B_\pr$, $I \models a$ if $a \in I$, otherwise $I \not\models a$; 
$I \models \dnot a$ if $I \not\models a$;
for a set of atoms $S$, $I \models S$ if $I \models l$ for all $l \in S$, otherwise $I \not\models S$;
for a rule $r \in \gpr$, $I \models r$ if $I \cap H(r) \neq \emptyset$ whenever $I \models B(r)$;
for a program $\gpr$, $I \models \gpr$ if $I \models r$ for all $r \in \gpr$.
$I$ is a \emph{model} of $\gpr$ if $I \models \gpr$.

The \emph{reduct} $\gpr_I$ of a program $\gpr$ with respect to an interpretation $I$ is obtained from $\gpr$ as follows:
(i) any rule $r$ such that $I \not\models B^-(r)$ is removed;
(ii) any negated literal $l$ such that $I \not\models l$ is removed from the body of the remaining rules.
An interpretation $I$ is an \textit{answer set} (stable model) of a program $\gpr$ if $I \models \gpr$, and there is no $J \subset I$ such that $J \models \gpr_I$.
The set of all answer sets of $\pr$ is denoted by $AS(\pr)$. $\pr$ is \emph{incoherent}, if $AS(\pr) = \emptyset$, and \emph{coherent} otherwise.

%
%
%

\myParagraph{Support.}
Given a model $I$ for a ground program $\pr$, 
we say that a ground atom $a \in I$ is {\em supported}
with respect to $I$ if there exists a \emph{supporting} rule $r\in \gpr$
such that $I \models B(r)$, $I\models a$ and $I \not \models (H(r) \setminus \{a\}$).
As it follows from the definition of the semantics given above, all atoms in an answer set $I$ must be supported.

\section{Debugging Approach}
\label{sec:debugging-approach}
In the following we present our approach to interactive localization of faults in non-ground ASP programs.
In general, one can differentiate between syntactic and semantic faults which require specific methods for debugging them. 
The first type of faults is usually detected by parsers of ASP grounders, whereas semantic faults can only be observed by a user while analyzing answer sets returned by a solver. 
In order to detect faults of the second type many ASP users verify the correctness of a program by testing it on a sample instance, which is common for software development.
In this case a user compares a (sub)set of all answer sets returned by a solver with expected solutions determined by hand. 
Therefore, often at least one answer set of the program for the sample instance is known to the user, otherwise it is impossible to understand that some program is buggy.
That is, a bug is then revealed when the known answer set is not among the computed ones.

\begin{definition}[Buggy Program]\label{def:buggyencoding}
	Let $\pr^c$ be the intended (correct) program that a user is going to formulate and $AS(\pr^c)$ be a set of its known answer sets.
	Then, a program $\pr$ is said to be \textit{buggy} with respect to a program $\pr^c$ if there exists an answer set $A \in AS(\pr^c)$ such that $A \not\in AS(\pr)$.
\end{definition}
Note that by this definition our approach deals only with situations in which some answer set of the correct program is missing. The opposite problem -- there is an answer set $A \in AS(\pr)$ such that $A \not\in AS(\pr^c)$ -- is not the focus of this paper. 

\begin{example}[Buggy Program]
	\label{ex:test-case}
	Consider the program $\pr'$ representing a (buggy) encoding for the graph coloring problem:
	\begin{equation*} 
		\begin{array}{l}	
			node(X) \imp edge(X,Y) \\
			node(X) \imp edge(Y,X) \\
			col(X, blue) \vee col(X, red) \vee col(X, green) \imp node(X) \\
			\imp col(X,C_1), col(Y,C_2),  edge(X,Y), X \neq Y, C_1 \neq C_2 \vspace{5pt}
		\end{array}
	\end{equation*}
	During the development of the encoding the user might create a simple graph, e.g. considering the sample instance comprising two facts:
	\begin{equation*}
		edge(1,2)\imp \quad \quad edge(2,3) \imp
	\end{equation*}
	For this instance the user expects the assignment of the $blue$ color to the nodes 1 and 3 as well as of the $red$ color to the node 2 to be among the solutions.
	However, the corresponding answer set encoding this solution is missing due to a bug in the encoding. In particular, note that the condition $C_1 \neq C_2$ should be replaced by $C_1 = C_2$.\hfill$\lhd$
\end{example}

The situation in which some solution is missing can be detected by means of testing, which is a common approach in software engineering aiming at identification and localization of faults in programs. 

\begin{definition}[Test Case]
	\label{def:test-case}
	Let $\pr^c$ be the intended program, $\pr$ be a program and $B_\pr$ be a Herbrand base of $\pr$. A set of atoms $T \subseteq B_\pr$ is a \emph{test case} for a program $\pr$ iff there exists an answer set $A \in AS(\pr^c)$ such that $T \subseteq A$.
\end{definition}

\begin{definition}[Test Case Failure]
	\label{def:test-case-failure}
	Given a program $\pr$ and a test case $T$, let $\pr_T = \{\imp \overline{l} \mid l \in T \}$, we say that $T$ \emph{fails} if $\pr \cup \pr_T$ is incoherent.
\end{definition}

Assertions of a test case are modeled by constraints that force the asserted atoms to be in all answer sets.
As a result, checking whether a test case $T$ of a program $\pr$ passes or not is reduced to checking whether $\pr \cup \pr_T$ is coherent, as illustrated in Example~\ref{ex:test-case-ctd}.

\begin{example}[Failing Test Case]
	\label{ex:test-case-ctd}
	Consider the program $\pr'$ from Example~\ref{ex:test-case} and the test case 
	\begin{equation*}
		T = \{col(1,blue), col(2,red), col(3,blue)\}.
	\end{equation*}
	The program $\pr'_{T}$ is composed by the constraints $\imp \dnot col(1,blue)$, $\imp \dnot col(2,red)$, and $\imp \dnot col(3,blue)$.
	Thus, $T$ is failing since $\pr' \cup \pr'_{T}$ is incoherent. \hfill$\lhd$
\end{example}

Whenever a test case fails, i.e.\ the given program $\pr$ is buggy, the goal of a debugger is to find an explanation for this observation.
However, in many cases it might be obvious to a user that some rules in $\pr$ are definitely correct and are not related to a fault, e.g., facts defining the test instance or some simple rules.
In such situations, the user might want to communicate this background knowledge to the debugger in order to exclude explanation candidates that are not explanations of the fault.
In practice, this allows the debugger to stay focused on the fault and reduce its runtime.

\begin{definition}[Background Knowledge]
	\label{def:back-know}
	Given a program $\pr$ the \emph{background knowledge} $\bg \subset \pr$ is a set of rules considered to be correct.
\end{definition}

\begin{example}[Background Knowledge]
	\label{ex:back-know}
	Consider the program $\pr'$ from Example~\ref{ex:test-case} and assume that the user is sure that the sample instance is encoded correctly.
	Therefore, the background knowledge $\bg'$ of $\pr'$ is composed by the facts $edge(1,2) \imp$ and $edge(2,3) \imp$. \hfill$\lhd$
\end{example}

Note that while working on various industrial applications and developing this debugging approach, we found that it is advantageous to move facts of test instances to the background knowledge.
This behavior is implemented as the default one in our debugger, unless the user provides a custom definition of the background knowledge.

In general, there are two possible causes for the incoherence of  $\pr \ \cup \ \pr_T$ considered in the literature on ASP debugging: (1) over-constrained programs and (2) missing support. 
In the first case we would like to find and highlight a set of rules in $\pr \setminus \bg$ that erroneously constrain the set of all answer sets and eliminate the intended ones. 
%
If the problem is due to the missing support, i.e.\ none of the rules in $\pr\setminus \bg$ allow for derivation of some atoms in the intended answer set, then we would like to highlight the corresponding atoms in the test case.

\begin{example}[Errors Detection]
	Consider the program $\pr'$ and the test case $T$ from Example~\ref{ex:test-case-ctd}.
	A debugger should identify the buggy rule:
	\begin{align*}
		\imp col(X,C_1), col(Y,C_2),  edge(X,Y), X \neq Y, C_1 \neq C_2 \enspace.
	\end{align*}
	Indeed, given the constraints $\imp \dnot col(1,blue)$ and $\imp \dnot col(2,red)$ in $\pr'_T$, the above constraint cannot be satisfied.
	In that case, the condition $C_1 \neq C_2$ should be replaced by $C_1 = C_2$. \hfill$\lhd$
\end{example}

In our approach, fault identification is done by constructing a specific program that allows a solver to find rules and/or atoms explaining the fault.
Note that this program is different from the one generated by meta-programming debuggers, since it uses no reification (see Sections \ref{sec:experiments} and \ref{sec:related} for more details).
Instead, we only extend the buggy program in a way that allows the debugger to map its results back to the input program.

\begin{definition}[Debugging Program]
	\label{def:debugging-program}
	Let $\pr$ be a program, $\bg$ be the background knowledge, and $id : (\pr \setminus \bg) \to \mathbb{N}$ be an assignment of unique identifiers to the non-background knowledge rules of $\pr$.
	Then, given 
	\begin{enumerate}
		\item the program
		$
		\dpr[D]{\pr} = \{H(r) \imp B(r) \cup \{\debug{id(r)}{\vec{\mathit{vars}}} \} \mid r \in (\pr \setminus \bg)\}
		$,
		where  $\debug{id(r)}{\vec{\mathit{vars}}}$ is a fresh atom and $\vec{\mathit{vars}}$ is a tuple with all variables of $r$, and 
		\item the program
		$
		\dpr{\pr}^S = \{ a \imp \dnot \support{a} \mid a \in B_{\pr} \}
		$, 
		where $\support{a}$ is a fresh atom called \emph{supporting atom} of $a$;
	\end{enumerate}
	the \emph{debugging program} $\dpr{\pr}$ of $\pr$ is defined as  $\dpr{\pr} = \dpr[D]{\pr} \cup \dpr{\pr}^S \cup \bg$.
\end{definition}

\begin{example}[Debugging Program]
	\label{ex:debugging-program}
	Consider the program $\pr'$ and the background knowledge $\bg'$ from Example~\ref{ex:back-know}.
	The \emph{debugging program} $\dpr{\pr'}$ is the following set of rules:
	\begin{equation*}
		\begin{array}{lll}
			\dpr[D]{\pr}: & \multicolumn{2}{l}{node(X) \imp edge(X,Y), \debug{1}{X,Y}} \\
			~& \multicolumn{2}{l}{node(X) \imp edge(Y,X), \debug{2}{Y,X}}\\
			~ &\multicolumn{2}{l}{col(X, blue) \vee col(X, red) \vee col(X, green) \imp node(X), \debug{3}{X}} \\
			~ &\multicolumn{2}{l}{\imp col(X,C_1), \ col(Y,C_2), \ edge(X,Y),}\\ 
			~ &\multicolumn{2}{c}{ X \neq Y, \ C_1 \neq C_2, \ \debug{4}{X,Y,C_1,C_2}}\\
			&&\\
			\dpr{\pr}^S: & node(i) \imp \dnot\support{node(i)} &\forall\ i \in \{1,2,3\}\\
			~ & edge(i,j) \imp \dnot\support{edge(i,j)} &\forall\ (i,j) \in \{(1,2), (2,3)\}\\
			~ & col(n,c) \imp \dnot\support{col(n,c)} & \forall\ n \in \{1,2,3\}, \forall\ c \in \{blue,red,green\}\\
			&&\\
			\bg: & edge(1,2) \imp & edge(2,3) \imp\hfill\lhd 
		\end{array}
	\end{equation*}
\end{example}

Since atoms of the form $\debug{id(r)}{\vec{\mathit{vars}}}$ and $\support{a}$ only appear in the body of the rules of $\dpr{\pr}$, they are not supported. Therefore, $\dpr{\pr}$ is extended to provide a supporting rule for all atoms of this form.

\begin{definition}[Extended Debugging Program]\label{def:ext-debugging-program}
	Let $\dpr{\pr}$ be a debugging program, $\mathcal{A}^D=\{\debug{id(r)}{\vec{\mathit{vars}}}  \ \mid \debug{id(r)}{\vec{\mathit{vars}}} \in B_{\dpr{\pr}}\}$ and $\mathcal{A}^S=\{\support{a} \ \mid \support{a} \in B_{\dpr{\pr}}\}$.
	Then, an \emph{extended debugging program} $\dpr[*]{\pr}$ is defined as $\dpr{\pr} \cup \{ \{a\} \imp \; \mid a \in (\mathcal{A}^D \cup \mathcal{A}^S)\}$, where $\{a\} \imp$ denotes a choice rule  \cite{DBLP:journals/ai/SimonsNS02}.
\end{definition}

\begin{example}[Extended Debugging Program]
	\label{ex:ext-debugging-program}
	Consider $\dpr{\pr'}$ from Example~\ref{ex:debugging-program}.
	Given $\dpr{\pr'}$ an intelligent grounder would output a program that comprises only the ground rules derived from the background knowledge, namely $edge(1,2) \imp$ and $edge(2,3) \imp$. All other rules will be dropped because of the atoms over $\_debug$ and $\_support$ predicates.
	
	The extended debugging program $\dpr[*]{\pr'}$ comprises the following additional set of rules:
	\begin{equation*}
		\begin{array}{ll}
			\{\debug{1}{i,j}\}  \imp & \forall\ (i,j) \in \{(1,2), (1,3)\}\\
			\{\debug{2}{i,j}\} \imp &\forall\ (i,j) \in \{(1,2), (1,3)\}\\
			\{\debug{3}{i}\} \imp& \forall\ i \in \{1,2,3\}\\
			\{\debug{4}{1,2,c_1,c_2}\} \imp &\forall\ c_1, c_2 \in \{blue,red,green\} \mid c_1 \neq c_2 \\
			\{\support{node(i)}\} \imp&\forall\ i \in \{1,2,3\}\\
			\{\support{edge(i,j)}\} \imp&\forall\ (i,j) \in \{(1,2), (1,3)\}\\
			\{\support{col(n,c)}\} \imp& \forall\ n \in \{1,2,3\}, \forall\ c \in \{blue,red,green\} \qquad \hfill\lhd\\
		\end{array}
	\end{equation*}
	These rules provide the necessary support to the fresh body atoms of the rules in  $\dpr[D]{\pr}$ and $\dpr{\pr}^S$ thus disabling the simplifications of a grounder.
\end{example}

It is important to show that the extended debugging program preserves some properties of the original program. In particular, in the following it is shown that, under some conditions, the extended debugging program is coherent if and only if the original program is coherent.

\begin{proposition}
	Let $\pr$ be a program, $\bg$ a background knowledge and $T$ a test case. 
	In addition, let $\pr_\mathcal{A} = \{ a \imp \mid a \in (\mathcal{A}^D \cup \mathcal{A}^S)\}$.
	Then, a program $\Gamma_\pr = \dpr[*]{\pr} \cup \pr_T \cup \pr_\mathcal{A}$ is coherent iff $\pr \cup \pr_T$ is coherent.
\end{proposition}

\begin{proof}[Proof sketch]
	The proof follows from the observation that the set of facts $\pr_\mathcal{A}$ in the program $\Gamma_\pr$ reduces $\dpr[*]{\pr}$ to $\pr$. 
	Namely, the set of rules $\{ \{a\} \imp \; \mid a \in (\mathcal{A}^D \cup \mathcal{A}^S)\}$ is trivially satisfied given  $\pr_\mathcal{A}$ and can be removed from consideration. Moreover, all atoms over $\_debug$ predicate must be valuated to true because of $\pr_\mathcal{A}$ and can be removed from bodies of corresponding rules.
	Finally, none of the bodies of rules in $\dpr[S]{\pr}$ are satisfied given $\pr_\mathcal{A}$ and, therefore, these rules are also removed.
\end{proof}

Consequently, checking the correctness of a test case $T$ of a program $\pr$ and background knowledge $\bg$ can be done by verifying if $\Gamma_\pr$ is coherent. In case $\Gamma_\pr$ is incoherent, and so is the $\pr \cup \pr_T$, we can use $\Gamma_\pr$ to find the reason of the incoherence.

\begin{definition}[Reason of incoherence]
	Let $\Gamma_\pr$ be an incoherent program. A set of rules $\mathcal{R} \subseteq \pr_\mathcal{A}$ is a reason of incoherence for $\Gamma_\pr$ if $(\Gamma_\pr \setminus \pr_\mathcal{A}) \cup \mathcal{R}$ is incoherent.
	A reason of incoherence $\mathcal{R}$ is \textit{minimal} if there is no set of rules $\mathcal{R}' \subset \mathcal{R}$ such that $\mathcal{R}'$ is reason of incoherence for $\Gamma_\pr$.
\end{definition}

\begin{example}[Reason of incoherence]
	Consider the extended debugging program $\dpr{\pr'}^*$ from Example~\ref{ex:ext-debugging-program} and the following test case:
	\begin{equation*}
		T = \{col(1,blue), col(2,red), col(3,blue)\}.	
	\end{equation*}
	Note that $\Gamma_{\pr'} \setminus \pr'_\mathcal{A}$ is coherent, whereas $\Gamma_\pr$ is incoherent. Thus, a (trivial) reason of incoherence would be the whole set $\pr'_\mathcal{A}$.
	There are two minimal reasons of incoherence, i.e. $\mathcal{R}_1 = \{\debug{4}{1,2,blue,red} \imp\}$ and $\mathcal{R}_2 = \{\debug{4}{2,3,blue,red} \imp\}$. Clearly, when atoms $col(1,blue)$ and $col(2,red)$ are true, the rule 
	\begin{equation*}
		\imp col(X,C_1), col(Y,C_2),  edge(X,Y), X \neq Y, C_1 \neq C_2	\end{equation*}
	with the instantiation $X=1$, $Y=2$, $C_1=blue$ and $C_2=red$ is violated, thus the atom $\debug{4}{1,2,blue,red}$ cannot be true.
	Moreover, note that both reasons originate from the same non-ground rule and are due to the symmetry of substitutions.$\hfill \lhd$
\end{example}
One important property of reasons of incoherence is their monotonicity, i.e. if a set of rules $\mathcal{R} \subseteq \pr_\mathcal{A}$ is a reason of incoherence, then all supersets $\mathcal{R}_1 \subseteq \pr_\mathcal{A}$ of $\mathcal{R}$ are also reasons of incoherence.

\begin{theorem}[Monotonicity]\label{thm:monotonicity}
	Let $\pr$ be a program, $T$ a test case, $\mathcal{B}$ a background knowledge, and $\dpr[*]{\pr}$ an extended debugging program over $\pr$ and $\mathcal{B}$.
	Let $\Gamma_\pr = \dpr[*]{\pr} \cup \pr_T \cup \pr_\mathcal{A}$ be an incoherent program and $\mathcal{R} \subseteq \pr_\mathcal{A}$ be a reason of incoherence, i.e.\ $(\Gamma_\pr \setminus \pr_\mathcal{A}) \cup \mathcal{R}$ is incoherent by definition.
	Then, any set of rules $\mathcal{R}_1$, such that $\mathcal{R} \subset \mathcal{R}_1 \subseteq \pr_\mathcal{A}$, is a reason of incoherence.
\end{theorem}

\begin{proof}
	Let $\mathcal{P}=\Gamma_\pr \setminus \pr_\mathcal{A}$.
	Suppose that $\mathcal{P} \cup \mathcal{R}_1$ is coherent and $M_1$ is an answer set.
	We will prove that $M_1$ is an answer set of $\mathcal{P} \cup \mathcal{R}$.
	Therefore, we have a contradiction.
	
	Let $\mathcal{R}_2 = \mathcal{R}_1 \setminus \mathcal{R}$ and
	let $\mathcal{A}_{\mathcal{R}_2} = \{a \mid a \imp \ \in \mathcal{R}_2\}$.
	For each atom $a \in \mathcal{A}_{\mathcal{R}_2}$ (i.e. of the form \debugempty or \supportempty) there is a choice rule of the form $\{a\} \imp \ \in \mathcal{P}$.
	Therefore, since $\{a\} \imp$ is the only rule containing $a$ in the head the following property holds:
	\begin{equation*}
		\begin{array}{c}
			AS(\mathcal{P} \cup \mathcal{R})\\
			=\\
			AS(\mathcal{P} \cup \mathcal{R} \cup \{\{a\} \imp \mid a \in \mathcal{A}_{\mathcal{R}_2}\})\\
			\supseteq\\
			AS(\mathcal{P} \cup \mathcal{R} \cup \{a \imp \mid a \in \mathcal{A}_{\mathcal{R}_2}\})\\
			=\\
			AS(\mathcal{P} \cup \mathcal{R} \cup \mathcal{R}_2)\\
			=\\
			AS(\mathcal{P} \cup \mathcal{R}_1)\\
		\end{array}
	\end{equation*}
	Therefore, $AS(\mathcal{P} \cup \mathcal{R}) \supseteq AS(\mathcal{P} \cup \mathcal{R}_1)$.
	Thus, if $M_1$ is an answer set of $\mathcal{P} \cup \mathcal{R}_1$ then $M_1$ is an answer set of $\mathcal{P} \cup \mathcal{R}$, which is impossible since by definition $\mathcal{P} \cup \mathcal{R}$ is incoherent. Consequently, $\mathcal{P} \cup \mathcal{R}_1$ is also incoherent.
\end{proof}

Note that there are multiple ways to prove Theorem~\ref{thm:monotonicity}. For instance, we can use the definition of reduct given in Section~\ref{sec:preliminaries}. 
The idea of the proof is based on the fact that for any test case $T$ the program $\pr_T$ defines a set of possible interpretations by constraining the truth assignments of atoms in $T$.
According to the definition of the reduct, for each interpretation $I$ there is only one reduct corresponding to it. 
Since bodies of all rules in the reduct are positive, i.e.\ comprise no negative literals, the consequence relation is monotonic. 
Therefore, we can always find at least one minimal reason of incoherence for every reduct and, consequently, for every interpretation allowed by the given test case.

Intuitively, a reason of incoherence represents a set of rules that makes the program incoherent. If the reason is minimal, removing one of those rules from the program makes it coherent.
Thus, debugging an incoherent program can be reduced to the process of finding a minimal reason of incoherence, fix it and then reiterate the process until all reasons have been analyzed.
However, in some cases a reason of incoherence might contain a large number of rules, thus making it infeasible to find the buggy rule among them.
Therefore, we aim at reducing the reasons of incoherence by querying the user on the atoms that must belong to the intended answer set.

\begin{example}[Buggy Encoding]
	\label{ex:test-case2}
	Consider the following program $\pr''$:
	\begin{equation*}
		\begin{array}{llll}
			a \imp c & \qquad
			b \imp \dnot c & \qquad
			c \imp \dnot b & \qquad
			\imp c, \dnot b
		\end{array}
	\end{equation*}
	and the test case $T=\{a\}$.
	The program $\dpr{\pr''}$ obtained from $\pr''$ is the following:
	\begin{equation*}
		\begin{array}{llll}
			a \imp c, \debugground{1} &
			b \imp \dnot c, \debugground{2}  &
			c \imp \dnot b, \debugground{3}  &
			\imp c, \dnot b, \debugground{4}\\
			a \imp \dnot \support{a} &
			b \imp \dnot \support{b} &
			c \imp \dnot \support{c}. & \\		
		\end{array}
	\end{equation*}
	In this case, $\mathcal{R} = \{ \debugground{4} \imp, \support{a} \imp, \support{b} \imp\}$ is a minimal reason of incoherence of the program $\Gamma_{\pr''}$. The intuitive meaning is that when the rule $\imp c, \dnot b$ is in the program the test case fails because $a$ and $b$ cannot be supported. Thus, the source of the error might be the rule $\imp c, \dnot b$ or one of the rules containing $a$ and $b$ in the head. 
	The idea is to query the user to reduce the possible source of errors.
	\hfill$\lhd$
\end{example}

\begin{definition}[Query]
	\label{def:query}
	Let $\Gamma_\pr$ be an incoherent program , let $T$ be a test case, and let $\mathcal{R}$ be a minimal reason of incoherence.
	A \emph{query} is an atom $q \in B_\pr \setminus T$.
	Let $\Gamma_\pr^* = (\Gamma_\pr \setminus  \pr_\mathcal{A}) \cup \mathcal{R}$, we define 
	\begin{equation*}
	\begin{array}{c}
	Q^+(q) = \bigcup_{r \in \mathcal{R}} \{ I \mid q \in I, I \in AS(\Gamma_\pr^* \setminus \{r\}) \} \\
	\mathrm{and}\\
	Q^-(q) = \bigcup_{r \in \mathcal{R}} \{ I \mid q \notin I, I \in AS(\Gamma_\pr^* \setminus \{r\}) \}.
	\end{array}
	\end{equation*}
\end{definition}
Note that if $\mathcal{R}$ is minimal, then $\Gamma_\pr^* \setminus \{r\}$ is coherent, for each $r \in \mathcal{R}$.
For a query atom $q$, the set $Q^+(q)$ contains all answer sets in which the query atom $q$ is true, whereas $Q^-(q)$ contains all answer sets in which $q$ is false.
Such sets are used to discriminate which atom is selected as query atom.
The user then should confirm whether the query atom $q$ is or not in the intended answer set.

\begin{example}[Query]
	\label{ex:query}
	Consider the program $\pr''$ from Example~\ref{ex:test-case2} and the minimal reason of incoherence $\mathcal{R} = \{ \debugground{4} \imp, \support{a} \imp, \support{b} \imp\}$. A query atom is one of $b$ and $c$.
	The program $\Gamma_\pr''^* \setminus \{\debugground{4} \imp\}$ admits the following answer sets:
	\begin{equation*}
		\begin{array}{l}
			I_1 = \{a, c, \support{a}, \support{b}, \debugground{1}, \debugground{3}\},\\
			I_2 = \{a, c, \support{a}, \support{b}, \debugground{1}, \debugground{2}, \debugground{3}\},\\
			I_3 = \{a,c, \support{a}, \support{b}, \support{c},\debugground{1}, \debugground{3}\},\\
			I_4 = \{a,c, \support{a}, \support{b}, \support{c}, \debugground{1}, \debugground{2}, \debugground{3}\},\\
			I_5 = \{a,c, \support{a}, \support{b}, \debugground{1}\},\\
			I_6 = \{a,c, \support{a}, \support{b}, \debugground{1}, \debugground{2}\}.
		\end{array}	
	\end{equation*}
	Next, the program $\Gamma_\pr''^* \setminus \{\support{a}\}$ admits the following answer sets: 
	\begin{equation*}
		\begin{array}{l}	
			I_7 = \{a, b, \support{b}, \support{c}, \debugground{2}, \debugground{4}\},\\
			I_8 = \{a, b, \support{b}, \support{c}, \debugground{2}, \debugground{3}, \debugground{4}\},\\
			I_9 = \{a, b, \support{b}, \support{c}, \debugground{1}, \debugground{2}, \debugground{4}\},\\
			I_{10} = \{a, b, \support{b}, \support{c}, \debugground{1}, \debugground{2}, \debugground{3}, \debugground{4}\},\\
			I_{11} = \{a, \support{b}, \support{c}, \debugground{4}\},\\
			I_{12} = \{a, \support{b}, \support{c}, \debugground{1}, \debugground{4}\}.
		\end{array}
	\end{equation*}
	Finally, the program $\Gamma_\pr''^* \setminus \{\support{b}\}$ admits the following answer sets: 
	\begin{equation*}
		\begin{array}{l}	
			I_{13} = \{a, b, c, \support{a}, \debugground{1}, \debugground{4}\},\\
			I_{14} = \{a, b, c, \support{a}, \debugground{1}, \debugground{2}, \debugground{4}\},\\
			I_{15} = \{a, b, c, \support{a}, \debugground{1}, \debugground{3}, \debugground{4}\},\\
			I_{16} = \{a, b, c, \support{a}, \debugground{1}, \debugground{2}, \debugground{3}, \debugground{4}\}. \phantom{\support{}}\\
		\end{array}
	\end{equation*}	
	Then, $Q^+(b) = \{I_7, \ldots, I_{10}, I_{13}, \ldots, I_{16}\}$ and $Q^-(b) = \{I_1, \ldots, I_6, I_{11}, I_{12}\}$, while 
	$Q^+(c) = \{I_1, \ldots, I_6, I_{13}, \ldots, I_{16}\}$ and $Q^-(c) = \{I_7, \ldots, I_{12}\}.$	$\hfill\lhd$
\end{example}
After computing the sets $Q^+(p)$ and $Q^-(p)$ for all atoms $p$, the idea is to select a query atom in a way that, regardless the answer to the query, the number of possible fixes is cut in half, i.e. the atom $q$ such that the absolute value of $|Q^+(q)| - |Q^-(q)|$ is minimum.
When the atom $q$ is selected, the user considers whether $q$ to be true in the expected answer set. If $q$ must be true then $\imp \dnot q$ is added to the extended debugging program, otherwise $\imp q$ is added.

\begin{example}[Query session]
Let us continue Example~\ref{def:query}. The atom $b$ is selected as query atom, since $|Q^+(b)| - |Q^-(b)| = 0$.
When the query $b$ is selected, the user considers whether $b$ to be true in the expected answer set. 
Assume the user selects $b$ to be true and $\imp \dnot b$ is added to $\dpr[*]{\pr''}$.
For the new version of the extended debugging program we compute the new minimal reason of incoherence $\mathcal{R} = \{\support{a} \imp, \support{b} \imp\}$. 
If the user answers the next only possible query $c$ with false, the  $\dpr[*]{\pr''}$ is extended with $\imp c$. 
Thus, the newly computed reason of incoherence comprises only one rule $\mathcal{R} = \{ \support{a} \imp\}$.
\end{example}

\section{A Debugger based on \dwasp}
\label{sec:gui}
In this section, a new graphical debugger based on \dwasp, called \dwaspgui, and its integration in \aspide~\cite{DBLP:conf/lpnmr/FebbraroRR11} are presented by running an example.
\subsection{The \dwasp Debugger}\label{sec:debugger}
Our implementation of the \dwasp debugger consists of two components: the debugging grounder \gringowrapper and \dwasp.
Figure \ref{fig:debugging-architecture} illustrates the interaction of both components to debug a program $\pr$.
First, the program $\pr$ is read by \gringowrapper from either the standard input or several input files.
The debugging grounder internally transforms the input program, passes the result to an ASP grounder   and outputs the ground debugging program to the standard output, which is then processed by \dwasp to start the interactive debugging session.
In general, \gringowrapper supports any \texttt{lparse}-compatible grounder, however, in our implementation we use \gringo~\cite{DBLP:conf/lpnmr/GebserKKS11}.

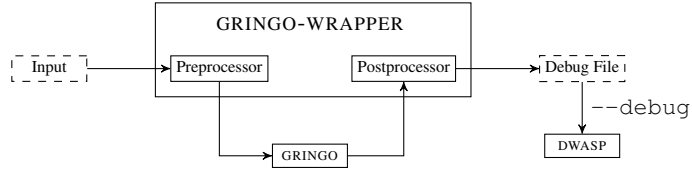
\begin{figure}[t!]
	\centering
\tikzstyle{block} = [draw, rectangle, minimum height=1.4em, minimum width=4.5em, node distance=3.5em,scale=0.7, transform shape]
\begin{tikzpicture}[->,>=stealth',auto,remember picture]
\node[block, dashed]                   (input)          {Input};
\node[block, right=of input]           (pre)            {Preprocessor};
\node[block, right=of pre]             (post)           {Postprocessor};
\node[block, dashed, right=of post]    (debug-file)     {Debug File};
\node[block,below=.69cm of debug-file] (wasp)           {\dwasp};
\path (pre) -- node[above=0.4cm]       (gringo-wrapper) {\gringowrapper} (post);
\path (pre) -- node[block,below=1cm]   (gringo)         {\gringo} (post);
\draw ($(pre.north west)+(-0.2,0.7)$) rectangle ($(post.south east)+(0.2,-0.2)$);
\draw (input)      -> (pre)        node[midway] {};
\draw (pre)        |- (gringo)     node[near end] {};
\draw (gringo)     -| (post)       node[near start] {};
\draw (post)       -> (debug-file) node[midway] {};
\draw (debug-file) -> (wasp)       node[midway] {\verb|--debug|};
\end{tikzpicture}
	\caption{Interaction of \gringowrapper and \dwasp in debugging mode.}
	\label{fig:debugging-architecture}
\end{figure}

\paragraph{Grounding with \gringowrapper.}
\label{sec:implementation:grounding}
The task of \gringowrapper is to obtain the grounded debugging program given an input program $\pr$ and some test case $T$.
First, $\pr$ is translated to the extended debugging program $\dpr[*]{\pr}$, as described in Definition \ref{def:ext-debugging-program} and $T$ is translated into $\pr_T$.
In case a user does not provide the background knowledge, by default, all facts of $\pr$ are assumed to be correct, i.e. the background knowledge $\bg$ comprises all facts of $\pr$.
After this transformation, \gringo is used to obtain the ground version of $\dpr[*]{\pr} \ \cup \pr_T$.
However, modern grounders perform several optimizations during grounding, such as deriving new facts from normal rules.
Although these optimizations potentially decrease the time required by the solver, they are counterproductive when debugging a logic program because wrong facts could be derived from faulty rules.
Moreover, a grounder might remove entire non-ground rules that are missing support.
In this case, \gringowrapper issues a warning message that highlights the rules that were removed by the grounder.

There are a number of ways to avoid the removal of atoms or simplification of rules done by grounders. One can use \texttt{-{}-keep-facts} option of \gringo (since version 4.5.4) or use the following workaround implemented in \gringowrapper:
First, the wrapper performs a call to the grounder and analyzes the produced \emph{atoms table} of the \texttt{lparse} format, i.e. a list of ground atoms occurring in the ground program. 
Then, for each atom $p$ the choice rule $\{p\}$ is added to the original program and the grounder is called again.
These additional rules are removed in a postprocessing step.

\paragraph{Debugging session with \dwasp.}
\dwasp is a specialized variant of \wasp~\cite{DBLP:conf/lpnmr/AlvianoDLR15}, a state-of-the-art ASP solver.
\wasp is based on a CDCL-like backtracking algorithm~\cite{DBLP:journals/tc/Marques-SilvaS99}, featuring the so-called \textit{incremental interface}~\cite{DBLP:conf/lpnmr/AlvianoDLR15}.
In particular, \wasp can take as input a ground ASP program $\Pi$ and a set of atoms $A$, called \textit{assumptions}, and computes either an answer set $I \supseteq A$ (if $\Pi$ is coherent) or a reason of incoherence $\mathcal{R} \subseteq A$ (if $\Pi$ is incoherent).
During a debugging session the solver is first invoked by providing as input
the program produced by the \gringowrapper, i.e. $\Pi = (\dpr{\pr}^* \ \cup \ \pr_T)^G$, and $A = \mathcal{A^D} \cup \mathcal{A^S}$.
In case an answer set is found the execution terminates and a message outlining the condition is provided to the user.
Otherwise, a reason of incoherence $\mathcal{R}$ is returned by \wasp.
Note that, as argued in \cite{DBLP:journals/tplp/AlvianoD16}, reasons of incoherence computed by \wasp are not minimal in general.
Therefore, \dwasp computes a minimal reason of incoherence $\mathcal{R}^*$ by using the state-of-the-art algorithm \textsc{quickxplain}~\cite{DBLP:conf/aaai/Junker04}.
$\mathcal{R}^*$ is then provided to the user.
Subsequently, \dwasp computes one or more query atoms according to the definitions of previous section.
Actually, in order to reduce the number of queries provided to the user, \dwasp implements a heuristic for the computation of query atoms. 
Given a minimal reason of incoherence $\mathcal{R}^*$, $\mid \mathcal{R}^* \mid$ calls to the solver are performed. In particular, for each element $p \in \mathcal{R}^*$ a call is performed where $\Pi = \dpr{\pr}^* \cup \pr_T$ and $A = \mathcal{R}^* \setminus \{p\}$. 
A heuristically-limited number of answer sets for each call is then used to compute an estimation of the sets $Q^+$ and $Q^-$ (note that the heuristic is needed to avoid to compute all possible answer sets; \citeN{DBLP:conf/lpnmr/DodaroGMRS15}). Then the queries are computed, and user may provide answers according to the expected solution. 
The answers are added in the solver input as described in previous section, and the process is repeated (computing a smaller reason of incoherence) until the bug is identified (or the user stops the debugger).

\subsection{The \dwaspgui}
\label{sec:system-description:dwaspgui}

\begin{figure}[t!]
	\centering
	\includegraphics[width=.8\textwidth]{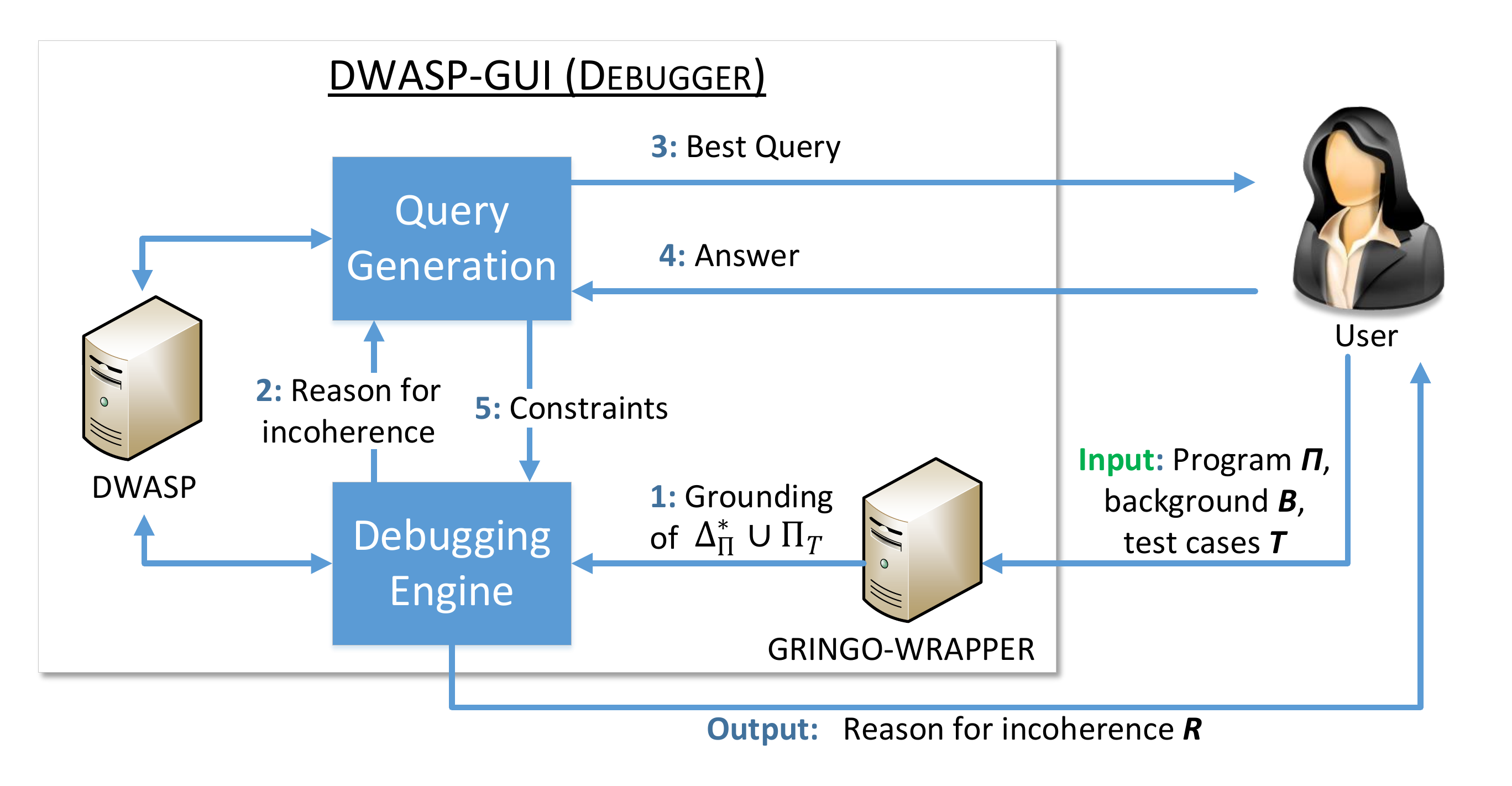}
	\caption{Interaction of the user with the debugging system: The front-end \dwaspgui uses \gringowrapper and \dwasp to debug the program.}
	\label{fig:architecture}
\end{figure}

\begin{figure}[t!]
	\centering
	\subfigure[Main window of the \dwaspgui.]{\label{fig:3col:main}
		\includegraphics[height=5.7cm, width=8.7cm]{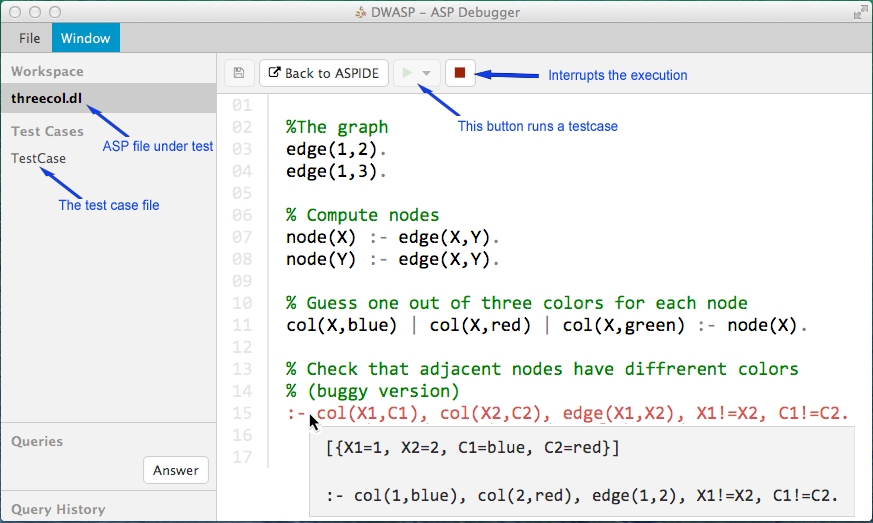}
	}
	\subfigure[A test case in \dwaspgui.]{\label{fig:3col:test}
		\includegraphics[height=5.7cm, width=3.5cm]{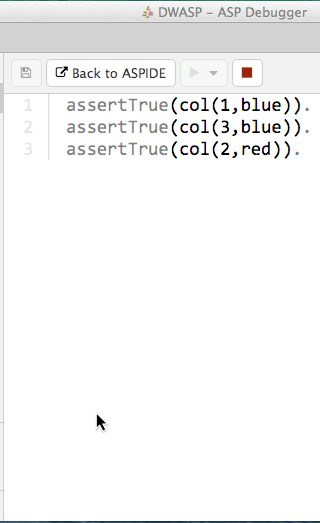}
	}	
	\subfigure[Unit testing and debugging (interaction with \aspide).]{\label{fig:3col:aspide}
		\includegraphics[width=0.99\textwidth]{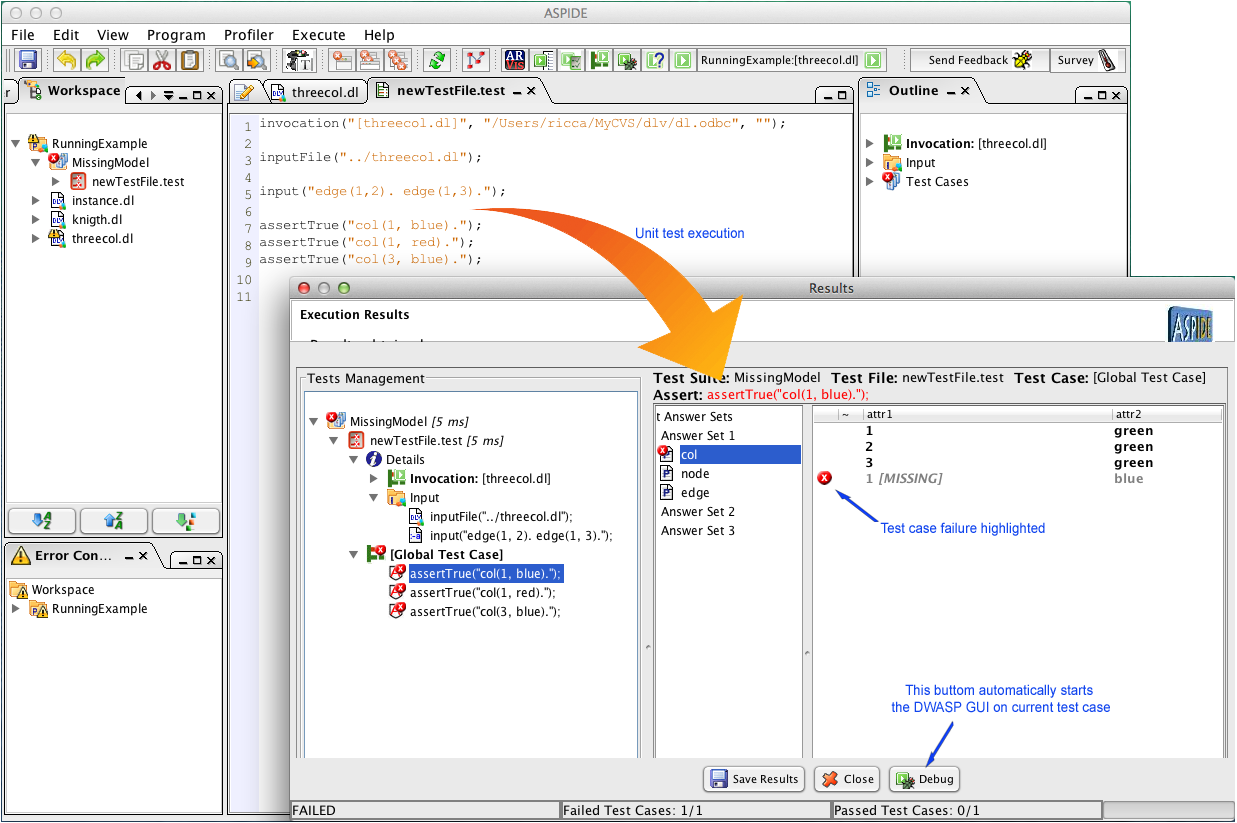}
	}
	\caption{Debugging and Testing the 3-Colorability encoding (Example~\ref{ex:test-case}).}
	\label{fig:dwasp-gui}
\end{figure}

\begin{figure}[t!]
	\centering
	\subfigure[Knight Tour (queries)]{\label{fig:knigth:query}
		\includegraphics[width=0.75\textwidth]{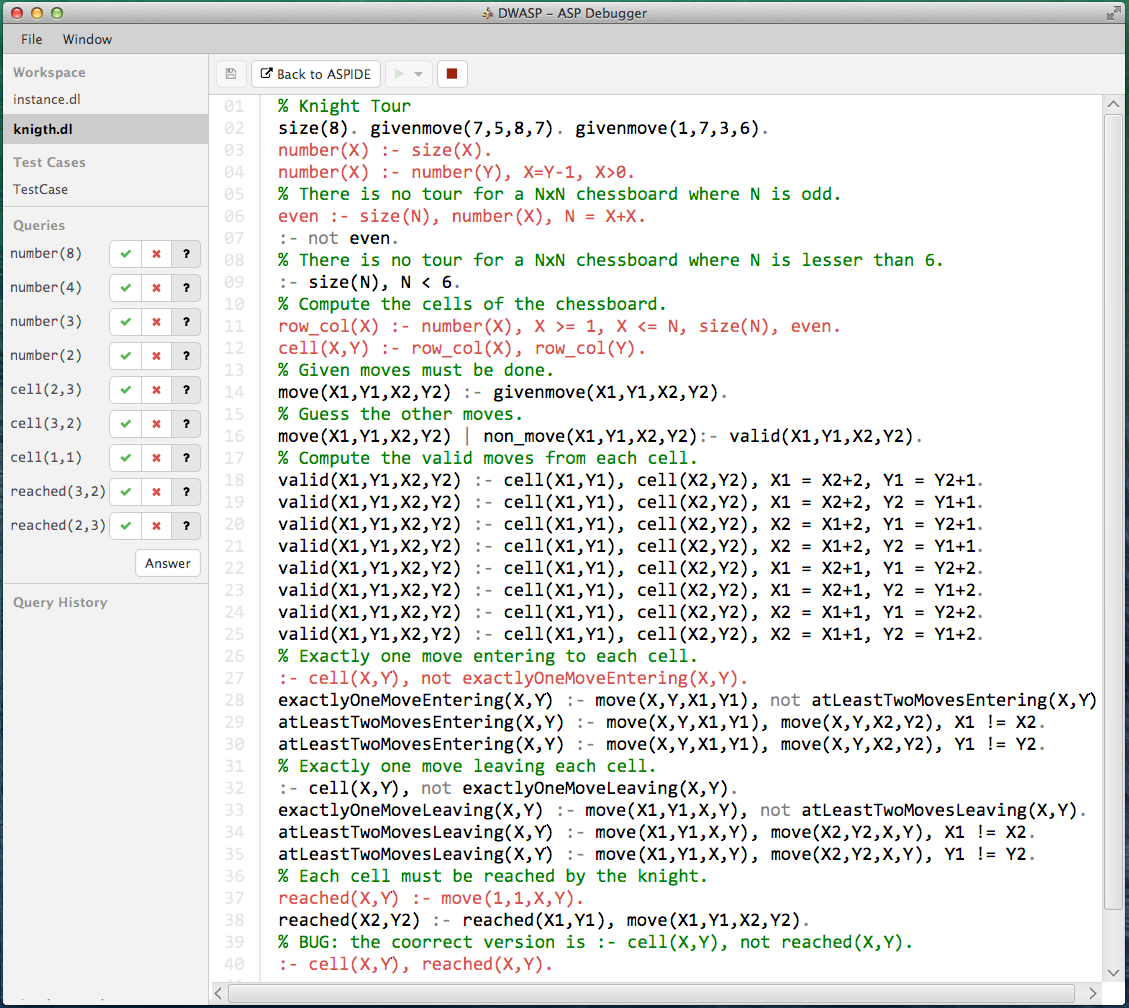}
	}
	\subfigure[Knight Tour (identified)]{\label{fig:knigth:answer}
		\includegraphics[width=0.75\textwidth]{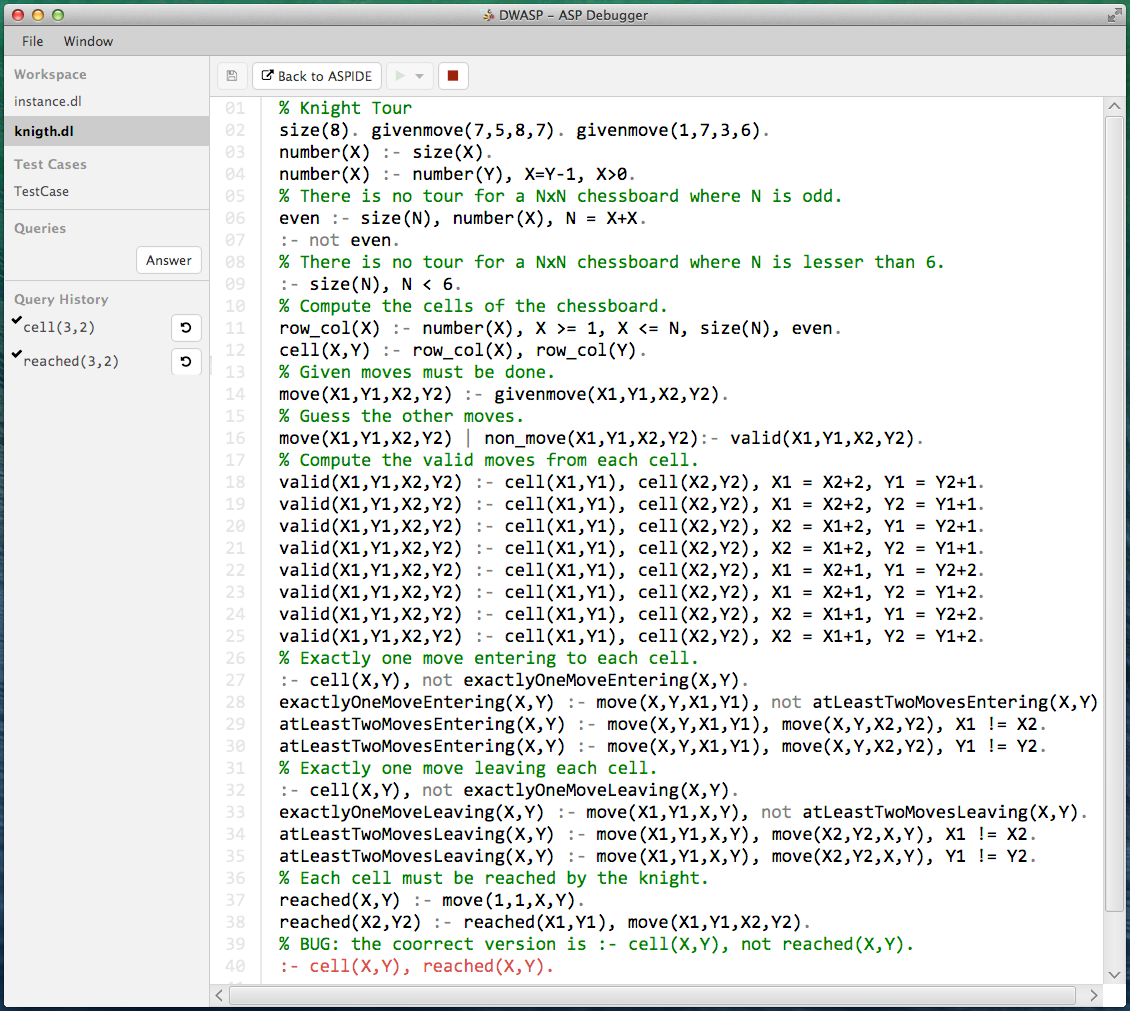}
	}
	\caption{Debugging the Knight Tour encoding from ASP Competition 2011.}
	\label{fig:dwasp-gui2}
\end{figure}
\paragraph{Architecture.}%
The architecture of the visual debugging system is depicted in Figure~\ref{fig:architecture}.
There are two main components: the \dwaspgui and
the debugger presented in the previous section. 
The \dwaspgui implements the graphical user interface, handles input and output, 
and controls the invocation of the debugger. 
In particular, the \dwaspgui wraps both \gringowrapper and \dwasp during the entire debugging session, so that the user can control them from a more friendly graphical environment. The components implement an interaction protocol that allows to exchange information and maintain the debugger in execution until it is interrupted by the user.
%
%
%


\paragraph{User Interface.}%
An instance of the \dwaspgui running Example~\ref{ex:test-case} is depicted in Figure \ref{fig:3col:main}).
The main window is split in two parts. 
The panels devoted to the specification of the inputs are on the left. There, the files containing the ASP program in input are listed below the label ``Workspace'', and the files containing test cases are listed below the label ``Test Case''. Indeed, the user can specify several test case for the same program, and the interface allows to debug one case at time.
Test cases are provided by the user as text files according to a simple syntax. 
For each atom $a$ that is expected to be true (resp. false) in the answer set, the user writes a statement \textit{assertTrue($a$)} (resp. \textit{assertFalse($a$)}).
The test case of Example~\ref{ex:test-case} was encoded as depicted in Figure~\ref{fig:3col:test}.
On the right middle part of the window (see Figure \ref{fig:3col:main}) there is a program editor featuring syntax highlights, where the user can edit both program and test case files. 
On top of the program editor is a tool-bar containing the buttons for running or concluding a debugging session on a specific test case. 
The button with a red square icon is used to stop a running debugger.
The test case to run can be selected from a drop-down list that is enabled by clicking on the run button (having a green triangle as icon). 
When a debugging session is started, \dwasp returns a minimal reason of incoherence $\mathcal{R}$, which is interpreted by the \dwaspgui by highlighting the corresponding non-ground rules in red. 
In particular, atoms of the form $\debug{i}{\dots}$ cause the highlight of the $i$-th rule, and when there are only atoms of the form $\support{a(\dots)}$, the rules having atoms of predicate $a$ in the head are highlighted.
The user can inspect a highlighted rule hovering over such a rule with the cursor, and the \dwaspgui shows shows in a pop-up both a substitution and a ground version of the rule causing the incoherence.
To have an idea of this functionality, the faulty constraints of
Example~\ref{ex:test-case} is shown highlighted and inspected by the user in Figure~\ref{fig:3col:main}.

The identification of the faulty constraints of Example~\ref{ex:test-case} is rather straightforward.
To illustrate how the interface handles more complex programs, and demonstrating the query feature of \dwasp, we purposely modified (introducing a simple bug in the last rule) the encoding of the \textit{Knight Tour} used in the ASP Competitions 2011.
The result obtained by running the debugger on a simple instance of that problem with the buggy encoding is depicted in Figure~\ref{fig:knigth:answer}. In this case, \dwasp identifies at first a number of rules as cause for the incoherence, but just one is guilty. 
At the same time \dwasp computes a set of possible queries to be answered by the user.
Queries are displayed in \dwaspgui on the left panel, and the user can answer yes (resp. no) by clicking on the green (resp. red) sign, or can leave the query unanswered. 
Note that the user can answer several queries at once and in random order, a feature not available in the command-line interface of \dwasp.
In our debugging session we answer that we expect \textit{cell(3,2)} and \textit{reached(3,2)} to be true, and as a result \dwasp is able to precisely identify the bug as depicted in Figure~\ref{fig:knigth:query}, and no further query can be posed to the user.
Note that, on the left \dwaspgui displays a query history where the answers given by the user can be inspected and possibly unrolled.

\begin{figure}[t!]
	\centering
	\subfigure[Run from the Main window]{\label{fig:aspide:run}
		\includegraphics[width=0.92\textwidth]{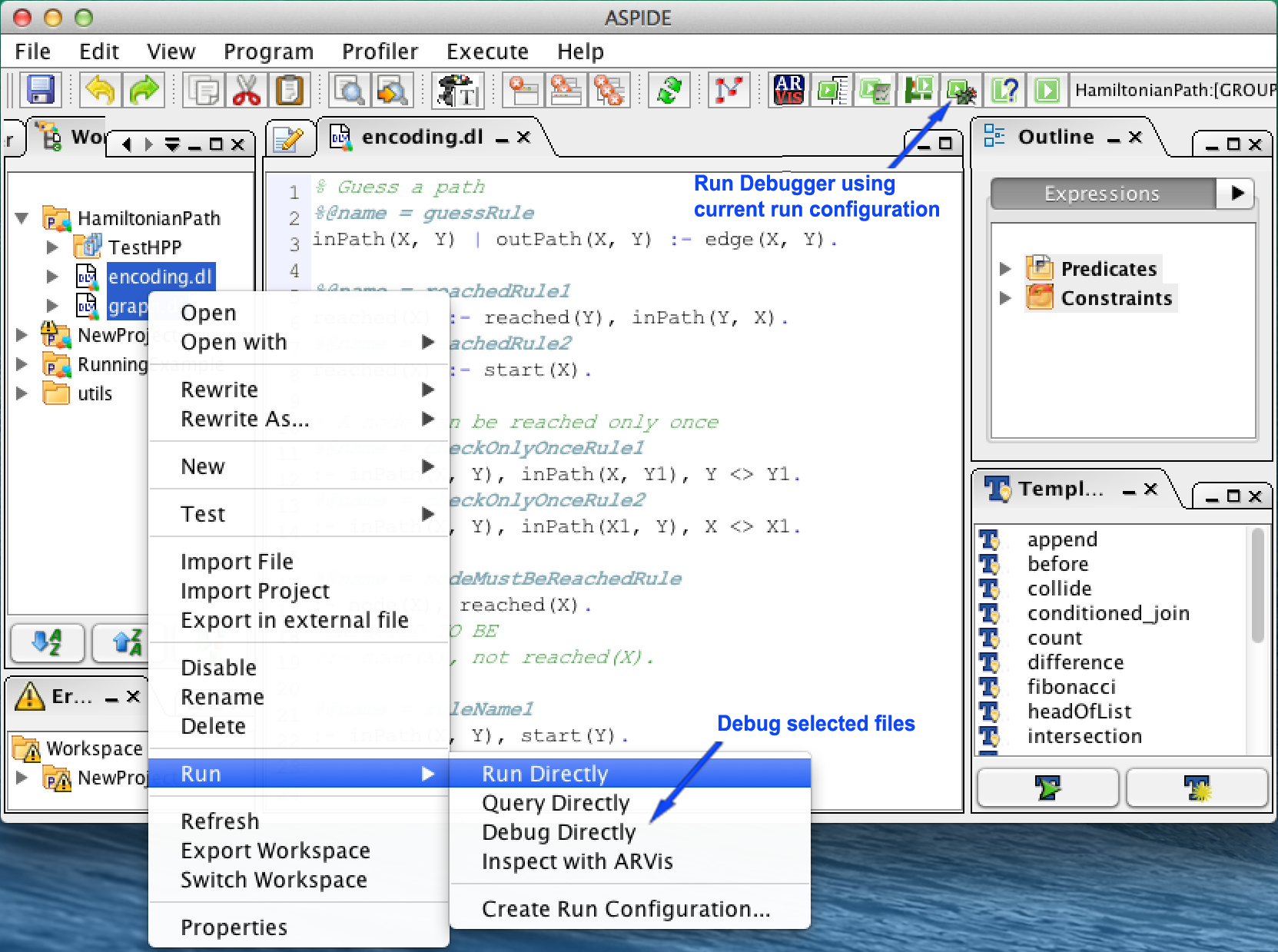}
	}
	\subfigure[Run from Results Window]{\label{fig:aspide:results}
		\includegraphics[width=0.92\textwidth]{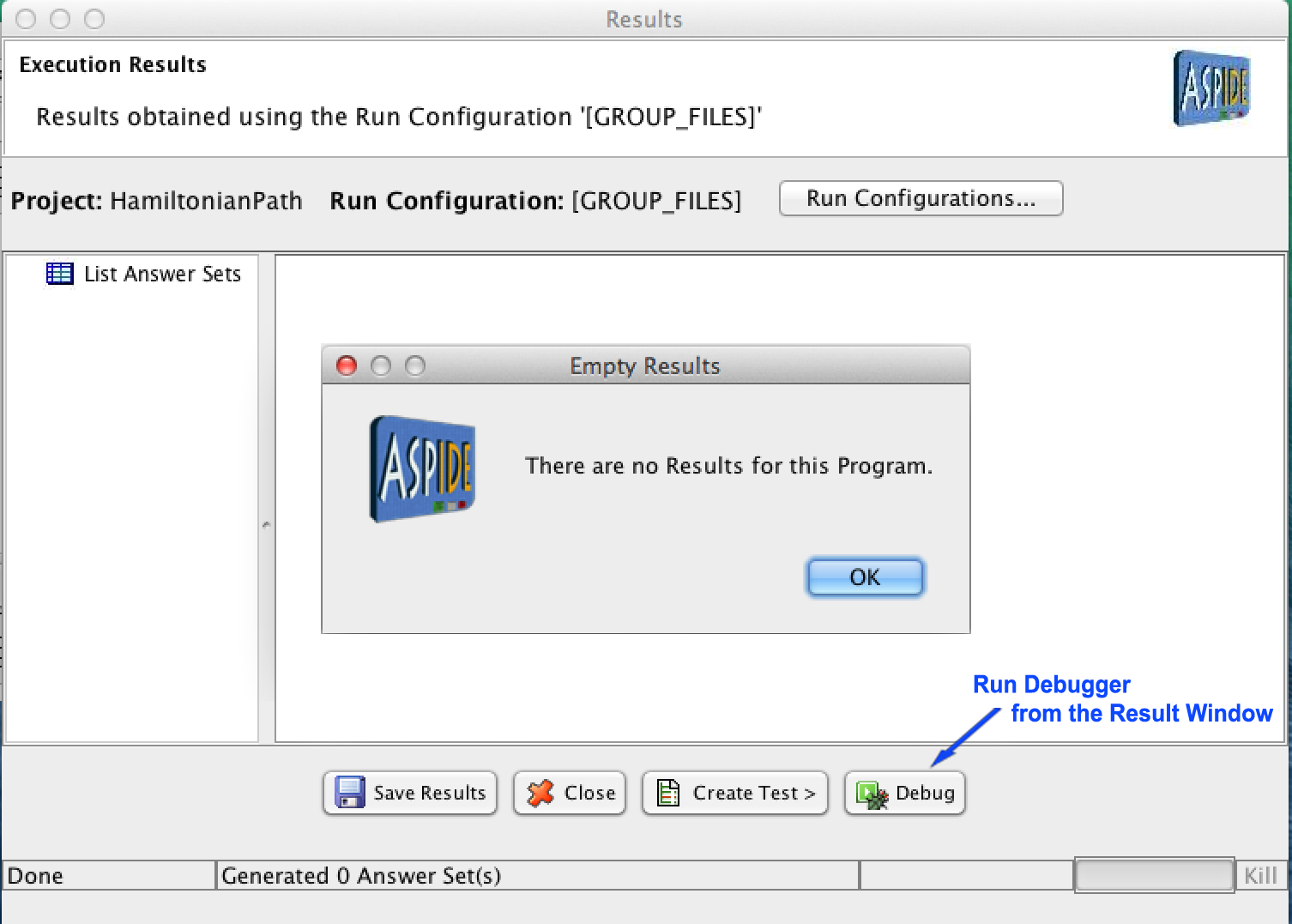}
	}
	\caption{Debugging the Knight Tour encoding from ASP Competition 2011.}
	\label{fig:aspideintegration}
\end{figure}
\subsection{Integration with \aspide}
\label{sec:system-description:aspide}
We integrated the graphical user interface \dwaspgui inside the integrated development environment \aspide \cite{DBLP:conf/lpnmr/FebbraroRR11} by developing a plug-in connector and extending some components of the user interface. 

As a result, \aspide offers several options for invoking the \dwaspgui. 
In the simplest scenario the user has to press a button in the main tool bar of \aspide.
This button (having a bug as icon) is pointed by a blue arrow in Figure~\ref{fig:aspide:run} and starts the debugger using current run configuration.
Alternatively, he/she can run the \dwaspgui on some specific files of a project.
Figure~\ref{fig:aspide:run} depicts also this use case, where the user (i) selects some files to debug in the project explorer (they are highlighted in blue on the left hand side), and (ii) right-clicks the mouse to select (iii) the menu item labeled "Debug Directly" (that starts the \dwaspgui).

Another handy option available in \aspide lets the user call the debugger directly from the window presenting the results of the execution of a solver. 
In Figure~\ref{fig:aspide:results} we see that current execution terminated and no answer set has been found, thus the user can start the debugger directly from that window by clicking on the dedicated button labeled "Debug" from the toolbar in the bottom left side of the window. 

In all the mentioned use cases the user has to do nothing for configuring \dwaspgui, 
all the needed intermediate files are automatically generated by \aspide, 
and the debugger is launched automatically on the specified program.

We now use a more general use case to present the combination of the new debugger with the unit testing framework of the IDE~\cite{DBLP:conf/inap/FebbraroLRR11}.
Unit testing is a white-box testing technique that requires to assess separately subparts of a source code called units to verify whether they behave as intended. 
\aspide supports a testing language that allows the developer to specify the rules composing one or several units, specify one or more inputs and assert a number of conditions on both expected outputs and the expected behavior of sub-programs. 
Test case specifications can be developed and run in \aspide, and the assertions are automatically verified by analyzing the output of the execution. 
\aspide provides the user with some graphic tools for simplifying the development and the inspection of results of test cases executions.%
\footnote{A complete description of the \aspide framework for unit testing is out of the scope of this paper, for more details we refer the reader to~\cite{DBLP:conf/inap/FebbraroLRR11}.}
In the \aspide testing tool one could easily identify a failing test case, but there was no support for understanding the cause of failure of a test case.
We solved this issue by connecting the debugger in the unit testing framework.

The work-flow for testing and debugging is now illustrated by using the program and the test case presented in Example~\ref{ex:test-case-ctd}.
In Figure~\ref{fig:3col:aspide}, we present a screen-shot of \aspide with a workspace that has the buggy graph colorability encoding loaded (see the file \textit{threecol.dl} on the left panel of \aspide). 
Test cases in \aspide can be defined according to a rich test case specification language that was introduced in~\cite{DBLP:conf/inap/FebbraroLRR11}, and that inspired the specification of test cases in \dwasp.
Actually, every \dwasp test case is also a valid unit test case specification for \aspide, modulo some additional syntactic construct needed to configure the testing tool with the program file to be tested and its input.
According to Example~\ref{ex:test-case-ctd} we defined the test case in file \textit{newTestFile.test}, 
and its specification is shown in the central editor window of \aspide of Figure~\ref{fig:3col:aspide}.
When test cases are executed, a new window is opened (small window in Figure~\ref{fig:3col:aspide}), where the result of the execution is shown.
Failing test cases are highlighted in red.
The user can start debugging of one of the failing test case by just clicking on the \emph{Debug} button, and the \dwaspgui window of Figure~\ref{fig:3col:main} is displayed.
Note that, once more the user has to do nothing for configuring \dwaspgui, 
all the needed files are automatically generated by \aspide.
Finally, the \emph{Back to ASPIDE} button allows to see the faulty rule highlighted also in \aspide.

\section{Performance Analysis}\label{sec:experiments}
We have assessed the performance of our implementation by comparing it with the debugger \textsc{Ouroboros}~\cite{DBLP:journals/tplp/OetschPT10,DBLP:conf/lpnmr/PolleresFSF13}, which is the only maintained solution able to cope with non-ground programs.
In particular we have employed the same ASP encodings and instances taken from ASP competitions that have been used in \cite{DBLP:conf/lpnmr/PolleresFSF13} for analyzing the performance of a debugger. The benchmark considered are Graph Colouring, Hanoi Tower, Knights Tour and Partner Units. 
For grounding we use \gringo (v4.4.0) in both methods.
The comparison is done by measuring grounding size and running time of \gringo for both debugging tools.

\begin{table}[b!]\label{tab:experiments}
	\caption{
		Comparison of grounding programs produced by \gringowrapper and \textsc{Ouroboros}.
	}
	\tabcolsep=0.050cm	
	\centering
	\begin{tabular}{lrrrrrrrr}
		\hline\hline
		& & & \multicolumn{2}{c}{\gringo} & \multicolumn{2}{c}{\gringowrapper} & \multicolumn{2}{c}{\textsc{Ouroboros}}\\
		\cmidrule{4-5}
		\cmidrule{6-7}
		\cmidrule{8-9}
		Benchmark	&	Instance	& \#n-ground	&	\#ground	&	time (s)	& \#ground &	time (s)	&	\#ground	&	time (s)\\
		\hline
		Graph Col	&	1-125	&	1\,672	&	6145	&	0.22	&	8\,031	&	0.63	&	19\,020	&	0.95\\
		Graph Col	&	11-130	&	1\,757	&	6455	&	0.21	&	8\,416	&	0.68	&	19\,845	&	1.10\\
		Graph Col	&	21-135	&	1\,986	&	7269	&	0.24	&	9\,305	&	0.73	&	21\,174	&	1.04\\
		Graph Col	&	30-135	&	1\,794	&	6597	&	0.25	&	8\,633	&	0.64	&	20\,502	&	1.02\\
		Graph Col	&	31-140	&	2\,039	&	7467	&	0.22	&	9\,578	&	0.67	&	21\,887	&	1.03\\
		Graph Col	&	40-140	&	2\,219	&	8097	&	0.32	&	10\,208	&	0.68	&	22\,517	&	1.03\\
		Graph Col	&	41-145	&	2\,262	&	8260	&	0.25	&	10\,446	&	0.68	&	23\,195	&	1.04\\
		Graph Col	&	51-120	&	2\,405	&	8773	&	0.36	&	11\,034	&	0.76	&	24\,223	&	1.05\\
		Hanoi	&	09-28	&	104	&	31\,748	&	0.40	&	94\,166	&	1.61	&	1\,739\,800	&	8.09\\
		Hanoi	&	11-30	&	106	&	34\,056	&	0.33	&	100\,942	&	1.58	&	1\,864\,222	&	9.50\\
		Hanoi	&	15-34	&	110	&	38\,672	&	0.38	&	114\,524	&	2.11	&	2\,112\,986	&	9.43\\
		Hanoi	&	16-40	&	100	&	27\,137	&	0.35	&	80\,615	&	1.40	&	1\,491\,281	&	7.04\\
		Hanoi	&	22-60	&	102	&	28\,311	&	0.29	&	84\,644	&	1.43	&	1\,678\,483	&	7.80\\
		Hanoi	&	38-80	&	106	&	34\,044	&	0.23	&	100\,942	&	1.68	&	1\,864\,250	&	8.53\\
		Hanoi	&	41-100	&	104	&	31\,738	&	0.39	&	94\,166	&	1.52	&	1\,739\,830	&	13.24\\
		Hanoi	&	47-120	&	99	&	25\,968	&	0.19	&	77\,227	&	1.49	&	1\,429\,695	&	6.90\\
		Knights Tour	&	01-08	&	21	&	1\,384	&	0.34	&	3\,413	&	1.14	&	12\,985\,716	&	59.44\\
		Knights Tour	&	03-12	&	22	&	3\,356	&	0.13	&	8\,652	&	0.60	&	>72\,244\,034	&	>300\\
		Knights Tour	&	05-16	&	21	&	6\,192	&	0.16	&	16\,285	&	0.64	&	>69\,494\,641	&	>300\\
		Knights Tour	&	06-20	&	21	&	9\,892	&	0.16	&	26\,321	&	0.88	&	>62\,785\,993	&	>300\\
		Knights Tour	&	07-30	&	21	&	22\,922	&	0.40	&	61\,911	&	1.13	&	>59\,166\,564	&	>300\\
		Knights Tour	&	08-40	&	21	&	41\,352	&	0.44	&	112\,501	&	1.27	&	>54\,944\,042	&	>300\\
		Knights Tour	&	09-46	&	21	&	55\,002	&	0.53	&	150\,055	&	1.58	&	>56\,443\,633	&	>300\\
		Knights Tour	&	10-50	&	22	&	65\,182	&	0.86	&	178\,094	&	2.15	&	>62\,402\,315	&	>300\\
		Partner Units	&	176-24	&	68	&	12\,563	&	0.22	&	14\,218	&	1.03	&	102\,023	&	1.47\\
		Partner Units	&	23-30	&	117	&	39\,231	&	0.29	&	42\,106	&	1.20	&	276\,645	&	2.11\\
		Partner Units	&	29-40	&	108	&	59\,979	&	0.34	&	64\,413	&	1.67	&	629\,639	&	3.35\\
		Partner Units	&	207-58	&	136	&	158\,564	&	0.61	&	168\,289	&	3.07	&	2\,726\,182	&	11.94\\
		Partner Units	&	204-67	&	141	&	218\,808	&	0.78	&	231\,083	&	5.30	&	4\,280\,282	&	17.79\\
		Partner Units	&	175-75	&	290	&	682\,015	&	2.10	&	699\,472	&	16.03	&	8\,604\,415	&	40.60\\
		Partner Units	&	52-100	&	254	&	952\,363	&	2.68	&	979\,603	&	16.61	&	20\,125\,857	&	90.10\\
		Partner Units	&	115-100	&	254	&	952\,369	&	2.86	&	979\,759	&	16.07	&	20\,317\,011	&	94.26\\
\hline\hline
	\end{tabular}
\end{table}
Results are reported in Table~\ref{tab:experiments}, where the first two columns represent the benchmarks and the instances considered, respectively. The third column is the number of rules of the non-ground program, while \#ground and time(s) are the size of the ground program and the execution time of \gringo, respectively.
In our approach the increase in grounding size is due to the fact that the \gringowrapper disables the optimizations performed by \gringo, whereas in \textsc{Ouroboros} the grounding of an ASP program modeling debugging is required.
Considering the instances that were groundable with \gringo within 5 minutes by our Intel Core i7-3667U machine with 8GB of RAM, we report that in our approach the size of the instantiation of the debugging program is from 1.5 to 3 times the size of grounding the original program, whereas the debugging program of \textsc{Ouroboros} generates groundings that are from 50 times up to 9382 times larger than the original program. 
Note that, the performance of our approach is only limited by the performance of the underlying solver, whereas in the case of \textsc{Ouroboros} the limit is in the grounding of the debugging program, which may not be feasible.

\section{Usability Test}
\label{sec:students}
In order to assess usability of the interface and degree of appreciation for our debugger tool we have set up a usability experiment. 
The test has been conducted during a regular class of the course on Answer Set Programming given by Prof. Nicola Leone for Bachelor Degree students in Computer Science at University of Calabria.
The assessment was executed on January, 19 2017 during a regular practice class at the end of the course.
The tool was not explained in a previous lecture and the experiment was not announced in advance to ensure that: 
$(i)$ users (i.e., students) never used the tool before, and 
$(ii)$ represent a sample of a distribution including both sufficiently skilled and less skilled ASP programmers;
$(iii)$ do not include only those that are interested in tools or have specific bias on using programming environments.

\paragraph{Test setup.} 
We prepared a test in which students were asked to find a bug on three ASP encodings.
We selected for this purpose tree well-known problems, and modified the encodings available from the third ASP competition~\cite{DBLP:journals/tplp/CalimeriIR14} website for the following problems: 3-Colorability, Hamiltonian Path, and Stable Marriage.
The first encoding is a classical example, which was familiar to the majority of users since it was presented during a lecture few months before to explain the guess and check methodology, and comprises only two (non ground) rules.
The encodings for the second and third problem were completely new to the audience, and they are much more complex featuring 6 rules each. In particular, the encoding of Stable Marriage is the least intuitive one and thus it was expected to be the hardest to fix.

We modified one constraint per encoding so that some expected answer sets were missing on a given test case.
All the encodings use basic features of the language, i.e, disjunctive normal logic programs as described in the founding paper by Gelfond and Lifschitz~\cite{DBLP:journals/ngc/GelfondL91} as well as in ASPCore 2.0 syntax~\cite{DBLP:journals/ai/CalimeriGMR16}. 
This choice ensures that the encodings and causes of faults in every test are comprehensible to the audience and no knowledge of advanced language constructs -- not covered by the lectures -- is required.
The students were provided with complete textual descriptions of the problems, including an explanation of the signatures and meaning of input and output predicates, one buggy encoding and one test instance per problem.
Student were working on own notebooks where \aspide with debugger was pre-installed and launched with a pre-loaded workspace containing one project per problem with all required files: problem description, encoding and sample instances. 

The test started after providing the users with: 
$(i)$ a description of the task to accomplish, 
$(ii)$ some minimal instructions on how to start the debugger and operate on the main buttons of the interface using a different example from the ones used in the test, and
$(iii)$ an anonymous questionnaire with a time annotating and debugger usage sheets for each problem.
In these sheets the students had to give notes on elapsed time and degree to which the debugger helped to find a bug.
The questionnaire, to fill at the end of the experience, contained the following questions:
\begin{itemize}
\item \textit{Do you find the debugger easy to use?}
\item \textit{Is bug detection faster using the debugger?}
\item \textit{Will you consider using it next time?}
\item \textit{How do you judge the user interface usability?}
\end{itemize}
The student could answer one of 
\textit{Strongly disagree, Disagree, Neither agree nor disagree, Agree, Strongly agree}  for the first three questions and one of \textit{Poor, Fair, Average, Good, Excellent} for the latter.

\newcommand{\larg}{6cm}
\begin{figure}[t!]
	\centering
	\subfigure[Do you find the debugger easy to use?]{\label{fig:usability:easy}
		\includegraphics[width=\larg]{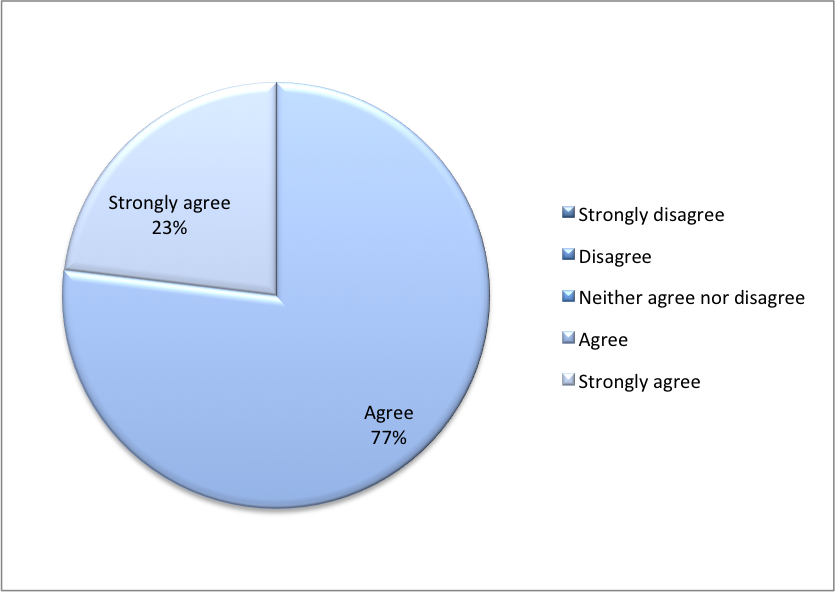}
	}
	\subfigure[Is bug detection faster using the debugger?]{\label{fig:usability:faster}
		\includegraphics[width=\larg]{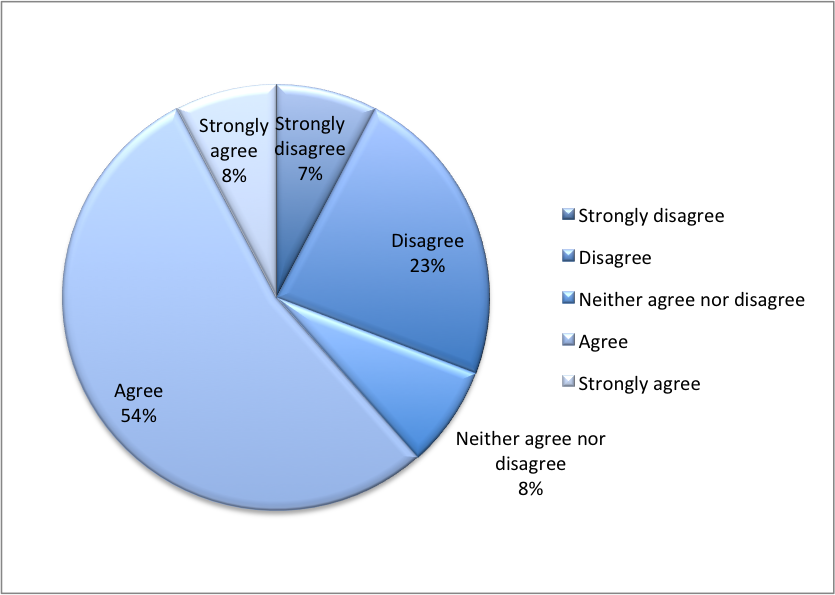}
	}
	\subfigure[Will you consider using it next time?]{\label{fig:usability:again}
		\includegraphics[width=\larg]{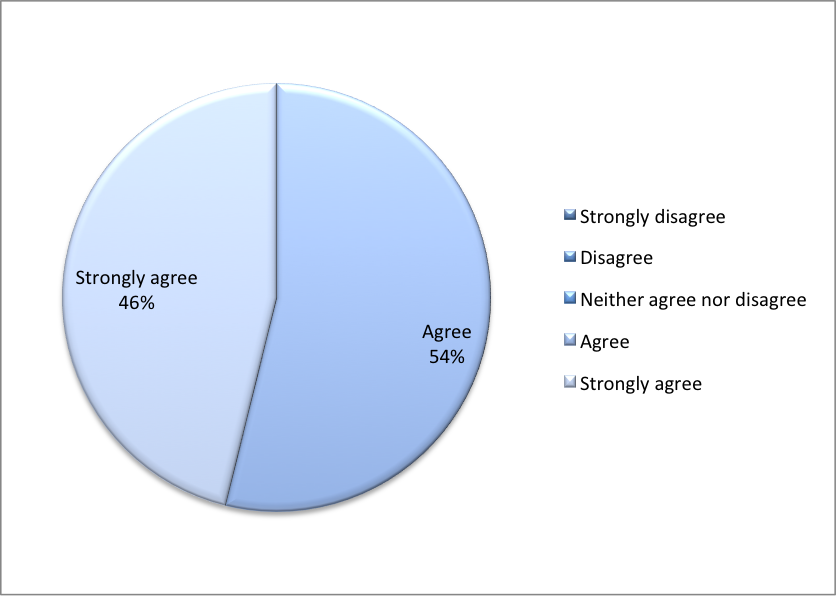}
	}	
	\subfigure[How do you judge the user interface usability?]{\label{fig:usability:overall}
		\includegraphics[width=\larg]{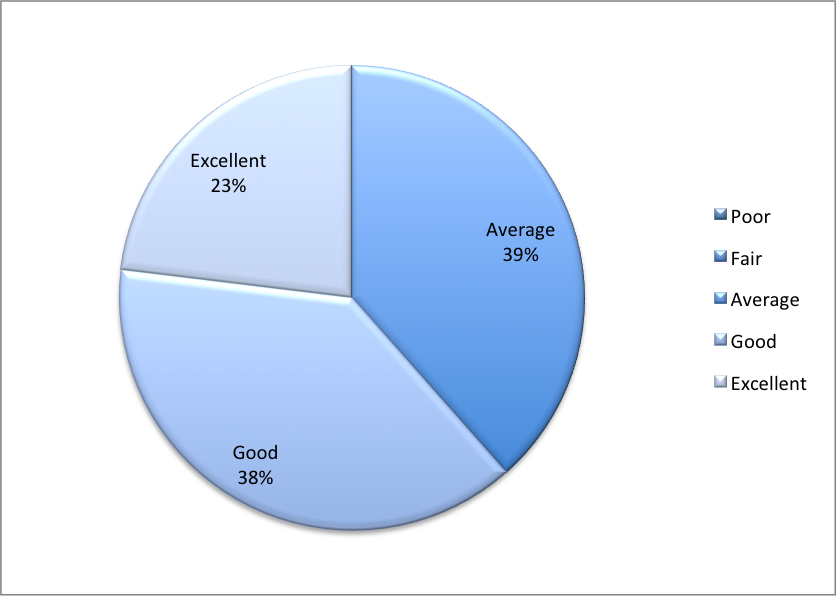}
	}	
	\caption{Results of the usability testing.}
	\label{fig:usability}
\end{figure}

\paragraph{Collection of results and hypothesis testing.}
We collected the results provided by 26 students on three usage tests (one per problem) of 30 minutes each. 
These number are in line with the Jakob Nielsen recommendation~\cite{DBLP:conf/chi/NielsenL93} for finding serious usability problems in user interfaces.%
\footnote{The Jakob Nielsen claim roughly says that few testers (no more than five users) and running as many small tests as you can afford is enough to identify a serious usability problem~\cite{DBLP:conf/chi/NielsenL93}.}
Moreover, to test the validity of our conclusions we applied the Kolmogorov Smirnov (K-S) test on the results, that refused the null hypothesis with an accuracy  $\geq$ 95\%.%
\footnote{The K–S test is one of the most useful and general nonparametric tests, that we used because it is more powerful than other methods (e.g. Chi-squared tests) when the size of the sample is below 50 elements, and some events (possible answers) have low frequency.}

\paragraph{Debugger applicability for complex problems.} 
One of our test goals was to determine the impact of bug fixing complexity on the applicability of the debugger. 
To verify that a problem is more difficult to solve than another we measured average bug fixing times on "fixed" cases as well as the number of cases in which a bug was identified.
3-Colorability -- the easiest problem -- was solved by all students but one in 6.4 minutes on average, and only 38.5\% declared the debugger was actually used for finding the bug. 
Hamiltonian Path required 9.1 minutes on average to be fixed, of which 91\% declared the debugger was used, and only one student failed the test.
Stable marriage was fixed in 8.2 minutes on average, and 100\% of the students declared the debugger was used, and two students failed the test.
For the sake of completeness, we observed that the student failing in 3-Colorability, failed also on Stable Marriage, and could solve Hamiltonian path in 25 minutes (the maximum, and clearly an outlier), thus we believe this was just a non proficient student with limited understanding of the language.
Thus, from this findings we conclude that the debugger was used more (and, thus empirically it was more useful) as the complexity of finding the bug increases.

\paragraph{Probing the opinion of the users.} 
Results of the questionnaire are summarized in Figure~\ref{fig:usability}.
All students agreed that the debugger is easy to use -- 77\% of students answered ``Agree'' and the 23\% ``Strongly agree'' to the first question as shown in Figure~\ref{fig:usability:easy}.
Concerning the second question (see Figure~\ref{fig:usability:faster}), 64\% conclude that using the debugger accelerates the bug fixing process (of which 8\% strongly agrees), 8\% is neutral, and 30\% disagrees. 
Interestingly, the last group includes all those students that failed at least one test as well as some of those that did not use the debugger for some test.
We interpret this result as follows: since the debugger was not needed to solve the easiest test, one cannot agree in general that it always makes bug fixing faster.
This explanation agrees with the message that comes from results to the third question presented in  Figure~\ref{fig:usability:again}. 
In this case all students would consider using the debugger again -- actually 46\% strongly agree.
This further outlines that also the critical users found it better to have our tool at their disposal. 
Finally, with the last question we asked whether they were satisfied with the user interface. 
The results are positive (see Figure~\ref{fig:usability:overall}): No student found the interface insufficient, 38\% claims it is a good interface, and 23\% finds it excellent.

We can conclude that our debugger was considered easy to use, and effective when the difficulty of bug-fixing is high; moreover, no serious usability problem was revealed, and the user interface was perceived to be largely acceptable.

\section{Related Work}
\label{sec:related}
There are multiple approaches to ASP debugging suggested in the literature including algorithmic~\cite{DBLP:conf/asp/BrainV05,Syrjanen2006}, stepping-based~\cite{DBLP:conf/lpnmr/OetschPT11} and meta-programming~\cite{Brain2007a,DBLP:conf/aaai/GebserPST08,DBLP:journals/tplp/OetschPT10,DBLP:conf/lpnmr/PolleresFSF13,DBLP:conf/aaai/Shchekotykhin15} methods. 
%
%
%
Among the algorithmic approaches \textsc{ideas}~\cite{DBLP:conf/asp/BrainV05} aims at explaining: 
(a) why a set of atoms $S$ is in an answer set $M$, and 
(b) why $S$ is not in any answer set. 
%
\textsc{ideas} allows a programmer: (1) to query for an explanation of an observed fault, (2) to analyze the obtained results and (3) reformulate the query to make it more precise. In our approach refinements are found automatically once the user provides additional knowledge on an expected answer set, thus, making the steps (2) and (3) obsolete.

Meta-programming debuggers use a program over a meta language -- a kind of ASP solver simulation -- to manipulate a program over an object language -- the faulty program. 
Each answer set of a meta-program comprises a \emph{diagnosis}, which is a set of meta-atoms describing the cause why some interpretation of the faulty program is not its answer set. 
The \textsc{spock}~\cite{DBLP:conf/aaai/GebserPST08} and \textsc{Ouroboros}~\cite{DBLP:journals/tplp/OetschPT10,DBLP:conf/lpnmr/PolleresFSF13} debuggers enable the identification of faults connected with over-constraint problems and unfounded sets. 
Both approaches represent the input program in a reified form allowing application of a debugging meta-program. In case of \textsc{spock} the debugging can be applied only to grounded programs, whereas \textsc{Ouroboros} can tackle non-grounded programs as well. 
Our approach does not fall in the meta-programming classification because it does not need any reification, nor a specific debugging program that manipulates the reified input program. 
These design choices are the main reason why meta-programming are affected by the grounding blowup (the grounding of the meta-program could be huge)~\cite{DBLP:conf/lpnmr/PolleresFSF13}.
Thus, the ground debugging program has to comprise all atoms explaining all possible faults in an input faulty program, which is not the case in our approach. 
Moreover, our approach generalizes the query-based method built on top of \textsc{spock}~\cite{DBLP:conf/aaai/Shchekotykhin15} by enabling its application to non-ground programs.
We also observe that our approach works in a radically different way with respect to meta-programming ones, since 
we just add a marker to each rule and compute (and minimize) reasons of incoherence.
Another difference is that \textsc{Ouroboros}, in case the bug is caused by an unfounded loop, is able to provide a loop comprising the atom. This information is missing in our approach, which just treats unfounded loops as missing support.

The approach of \textsc{smdebug}~\cite{Syrjanen2006} addresses debugging of incoherent non-disjunctive ASP programs by adaption of Reiter's model-based diagnosis. 
Similarly to our approach the debugger focuses on analyzing contradictions, but cannot detect problems arising due to some atom missing support (since only odd loops are considered to be errors). 

There are other approaches enabling faults localization in ASP, but not directly comparable with \dwasp, including Consistency-Restoring Prolog~\cite{Balduccini2003}, translation of ASP programs to natural language~\cite{DBLP:conf/icai/MikitiukMT07}, visualization of justifications for an answer set~\cite{DBLP:journals/tplp/PontelliSE09} as well as stepping through an ASP program~\cite{DBLP:conf/lpnmr/OetschPT11}. 
In \cite{DBLP:conf/iclp/LiVPSB15}, the authors present a debugging technique for normal ASP programs that is based on inductive logic programming (ILP) and test cases. The idea is to allow the programmer to specify test cases modeling features that are expected to appear in some solution and those that should not. 
These are used to to revise the original program semi-automatically so that it satisfies the stated properties.
This approach offers the possibility to learn rules (and modifications of rules), whereas \dwasp focus only on identifying the buggy rules of a given program.
Combining these approaches with ideas implemented in \dwasp is part of our future work.

In \cite{DBLP:conf/lpnmr/0001ST15,DBLP:journals/tplp/0001T16} bugs are studied in terms of a set of culprits (atoms) using semantics which are weaker than the answer set semantics. A technique for explaining the set of culprits in terms of derivations is also provided. 
Approaches explaining bugs with the truth of a set of atoms are, in a sense, complementary to our approach (we identify the rules involved in a conflict). 

In \cite{DBLP:journals/corr/DassevilleJ15} the web-based programming environment for the IDP system is presented that also features a debugging approach based on the computation of a reason of incoherence. This debugger does not feature a question-answering schema that is fundamental for reducing the set of buggy rules.
Moreover, we are not aware of any IDE for ASP that provides a tight combination of debugging and unit testing environments as the one presented in this paper.

\section{Conclusion}
\label{sec:conclusion}
ASP features an intuitive syntax and a well-known semantics, nonetheless the process of finding bugs in logic programs can be non trivial and is often a tedious task.
For this reason, valid ASP debuggers have emerged during the recent years. 
The most prominent approaches, using ASP itself to compute explanations, are however affected by two main issues somehow limiting their applicability some in practical cases: (i) the grounding blowup, that may make impossible to compute the causes of a bug; and, (ii) the overwhelming number of produced explanations, which might be impossible to be browsed by users.

In this paper we propose a novel debugging approach for non-ground ASP programs that is not affected by both the above issues.
Indeed, it points the user directly to a set of rules involved in the bug, and --importantly-- allows to refine that set interactively by asking the user specific questions on an expected answer set, until the bug can be easily identified.

The new approach has been implemented in the \dwasp Debugger, which was obtained by properly combining the grounder \gringo with an extended version of the ASP solver \wasp.
An empirical analysis shows that the new debugger is not affected by the grounding blowup, and can handle instances that are pragmatically out of reach for state-of-the-art meta-programming-based debuggers.

The \dwasp Debugger has been complemented by a user-friendly graphical interface, called \dwaspgui.
The graphical interface improves the user-experience of debugging ASP programs, as demonstrated by running a usability test on a class of students attending a university course on ASP.
Indeed, besides the usual advantages provided by visual tools, the \dwaspgui simplifies two tasks that are not easy to carry out in the command line interface, namely: the definition of test cases and the interactive query answering. The query answering feature is made much more user-friendly, since the user can simply select answers by clicking on dedicated buttons, and several possible answers are presented to the user in a convenient list.
Problematic rules are outlined immediately in the text editor so the user is pointed immediately from the interface to sources of bugs.
\dwaspgui has also been integrated in \aspide, which was missing a complete debugger interface supporting non ground ASP programs. The integration includes specific support for creating failing test cases to debug directly from the unit test framework provided by \aspide supporting test-driven development.
The rapid identification of the cause of a failig test case is fundamental for test-driven development~\cite{DBLP:conf/xpu/FraserBCMNP03}.
With our extension \aspide turns into a more complete IDE by offering improved debugging support and a more effective test-driven development environment.

Concerning future works, one possibility would be to study a possible integration of our approach with existing ones.
Moreover, an interesting work would be to investigate if our debugging approach can be generalized also in the case when an extra, incorrect, answer set is provided.
We also plan to extend the tool in order to better handle some specific bugs related to missing support, in particular those due to the so-called \textit{unfounded sets}.

\myParagraph{Availability.} 
The \dwaspgui can be obtained from~ \texttt{https://github.com/gaste/} \texttt{dwasp-gui}, and 
\aspide from \texttt{http://www.mat.unical.it/ricca/aspide}, the plugin connector installation starts the first time the debugger is launched.

\myParagraph{Acknowledgments.}
The authors are grateful to Marc Deneker and Ingmar Dasseville for the fruitful discussions about debugging for logic programs and FO(ID) theories, and in particular for the useful suggestion improving the handling of bugs caused by atoms missing a supporting rule. 
The authors are also grateful to Roland Kaminski for providing \gringo without simplifications.


\bibliographystyle{acmtrans}

\end{document}